\documentclass[twoside]{article}

\usepackage{amsmath,amssymb}
\usepackage{algpseudocode}
\usepackage{amsthm}
\usepackage{xcolor}
\usepackage{graphicx}
\usepackage[ruled,vlined]{algorithm2e}
\usepackage{hyperref}
\newtheorem{proposition}{Proposition}

\usepackage[accepted]{aistats2026}

\usepackage{fancyhdr}
\usepackage[round]{natbib}

\begin{document}

\twocolumn[

\aistatstitle{Active Measurement of Two-Point Correlations}

\aistatsauthor{ Max Hamilton \qquad Daniel Sheldon\qquad Subhransu Maji}

\aistatsaddress{\texttt{\{jmhamilton,sheldon,smaji\}@cs.umass.edu}\\University of Massachusetts, Amherst} ]

\begin{abstract}
Two-point correlation functions (2PCF) are widely used to characterize how points cluster in space. In this work, we study the problem of measuring the 2PCF over a large set of points, restricted to a subset satisfying a property of interest. An example comes from astronomy, where scientists measure the 2PCF of star clusters, which make up only a tiny subset of possible sources within a galaxy. This task typically requires careful labeling of sources to construct catalogs, which is time-consuming. We present a human-in-the-loop framework for efficient estimation of 2PCF of target sources. By leveraging a pre-trained classifier to guide sampling, our approach adaptively selects the most informative points for human annotation. After each annotation, it produces unbiased estimates of pair counts across multiple distance bins simultaneously. Compared to simple Monte Carlo approaches, our method achieves substantially lower variance while significantly reducing annotation effort. We introduce a novel unbiased estimator, sampling strategy, and confidence interval construction that together enable scalable and statistically grounded measurement of two-point correlations in  astronomy datasets. 
\end{abstract}

\section{INTRODUCTION}

The two-point correlation function (2PCF) is a statistical tool used to quantify how points are spatially distributed relative to one another. It has a rich history in astronomy, particularly in the study of the large-scale structure of the universe. Early studies \citep[e.g.,][]{peebles1975jagellonian} applied this technique to analyze how matter---especially galaxies---is distributed across cosmic scales. With the advent of high-resolution telescopes such as Hubble and the James Webb Space Telescope (JWST), 2PCF analysis has been extended to smaller scales, including sources within individual galaxies ~\citep[e.g., star clusters or stars;][]{Lapeer_2026, 2024ApJ...971...32P, 2017ApJ...840..113G,  2017ApJ...842...25G}.

Performing such analyses, however, requires corresponding catalogs of objects, which can be time-consuming to produce. Early galaxy surveys such as the Zwicky~\citep{zwicky1968catalogue} and Lick~\citep{seldner1977lick,shane1967lick} catalogs were monumental efforts---e.g., the Lick catalog involved manually counting approximately one million galaxies from photographic plates~\citep{Lickgalaxycatalogue}. Catalogs of star clusters similarly require classifying sources within galaxies based on their morphology and spectral characteristics. Large teams of astronomers have collaborated to compile such catalogs for several galaxies using data from Hubble surveys \citep[e.g., LEGUS;][]{2017ApJ...841..131A,2015AJ....149...51C,2018MNRAS.473..996M}. However, JWST and future telescopes can resolve far richer detail and vastly more objects, making manual cataloging increasingly challenging.

While AI techniques—particularly computer vision models—have been proposed for automated source classification~\citep{pmcmjas_apj2021,whitmore2021phangs}, astronomers are often reluctant to rely directly on model outputs. These models may be inaccurate or exhibit distribution shifts when applied to new galaxies, leading to biased estimates.

In this work, we introduce a human-in-the-loop framework for estimating the two-point correlation function (2PCF) of target sources using a small amount of labeled data. We assume that the target sources (e.g., star clusters) form a subset of all detected sources. Our method leverages a pre-trained classifier to identify the most informative sources for expert labeling. After each new label is acquired, it updates both an estimate and confidence intervals for pair counts across all distance bins—a key quantity required for 2PCF estimation. This enables astronomers to achieve the desired accuracy with substantially reduced annotation effort. Applied to real data from two galaxies observed in JWST surveys, our framework achieves less than 10\% error in the 2PCF using only 35\% of the objects labeled.

Our approach builds on ideas from subset selection and importance sampling, but adapts them to the specific challenges of the pair-counting problem. We introduce a new unbiased estimator for pair counts, an efficient method for selecting examples to label, and a technique for constructing confidence intervals across multiple separation bins. Together, these contributions enable large-scale, statistically robust 2PCF analyses with reduced manual effort.

Code for this paper is available at: 

\quad {\small\url{https://github.com/cvl-umass/2pcf-measurement}}

\section{RELATED WORK}

Our work is broadly related to methods that combine AI model predictions with human labels to improve estimates. Prediction powered inference~\citep{angelopoulos2023prediction} showed how to estimate a population parameter (e.g., sample mean or regression coefficient) from imperfect machine learning predictions by using a small labeled dataset to estimate model bias and construct confidence intervals. IS-count~\citep{Meng_Liu_Neiswanger_Song_Burke_Lobell_Ermon_2022} uses importance sampling to select a small subset based on covariates, while DISCount~\citep{Perez_Maji_Sheldon_2024} proposed using detector outputs as covariates to improve estimation accuracy. Recently, active measurement~\citep{hamilton2025active} combined these ideas with iterative model fine-tuning and adaptive importance sampling to reduce estimation error. All these methods, however, focus on estimating first-order statistics of the data (e.g., counts) and are not easily extended to pairwise measurements.

\cite{perez2024human} proposed an approach for producing unbiased estimates of connected components in a graph using a pairwise similarity measure derived from a classifier together with human labels. The approach is based on nested importance sampling, where nodes and edges are sampled using distributions derived from the pairwise similarity. This method was used to estimate population sizes from a collection of images of animals using a re-identification model. Similar to this work, we frame our problem as the task of counting edges that satisfy certain properties in a graph (see~\S~\ref{sec:setup}). However, there are significant differences both in the estimand (e.g., we are counting edges) and in the nature of the labeling task (e.g., labeling vertices vs. edges). As a result, a similar nested Monte Carlo approach performs poorly, and instead, we propose a novel set-based estimator that allows us to estimate the 2PCF more efficiently.

Our approach is related to the literature on subset selection, where the goal is to select a representative subset of a larger collection, often to improve training efficiency or performance. Combining subset selection with importance sampling has been explored in methods such as IWeS~\citep{citovsky2023leveraging}, which determines importance based on the entropy of model predictions. In contrast, our approach prioritizes subsets with large counts to reduce the variance of estimated totals. While we also use an AI model to assign importance weights, our primary goal is accurate estimation of counts rather than improved model accuracy.

Finally, our model builds on computer vision models for source classification in Astronomy. For example, data collected from the Galaxy Zoo project~\citep{raddick2009galaxy} has resulted in models~\citep{banerji2010galaxy} to classify galaxy types such as spirals, ellipticals, and so on. More recently, models such as StarcNet~\citep{pmcmjas_apj2021} and others~\citep{whitmore2021phangs} have been applied to classify source types within high-resolution images of galaxies acquired from the Hubble Space Telescope. While these models have accelerated astronomical analysis, they are far from perfect. For example, StarcNet, trained on a catalog of 50 galaxies from the LEGUS survey consisting of over 11k labeled sources, achieves only 80\% accuracy on the task of recognizing star clusters among other bright sources. Models trained on small, existing catalogs for James Webb Space Telescope (JWST) are even less accurate, yet the need to automate analysis is greater, as JWST can detect many more galaxies and sources within them. This motivates our effort of estimating the 2PCF with few expert labels, guided by model predictions.

\section{SETUP}\label{sec:setup}

\subsection{Two Point Correlation Function}

We adopt the angular two-point correlation function $\omega(\theta)$, which describes the excess probability, relative to a uniform random distribution, of finding two points separated by an angle $\theta$ at a given physical scale. This probability, measured within a solid angle $\delta \Omega$, is given by
\begin{equation}
    \delta P = N (1+\omega(\theta)]\delta \Omega,
\end{equation}
where $N$ is the average density of the point sample.

The 2PCF has been used in prior work to quantify the degree of clustering at different spatial scales in galaxies. For a uniformly distributed random sample, the 2PCF is flat, i.e., $1+\omega(\theta)=1$. For samples that exhibit clustering at small scales, we expect $1+\omega(\theta)>1$ at small $\theta$, indicating an excess of nearby pairs relative to random. At larger separations, one may instead observe $1+\omega(\theta)<1$, indicating fewer pairs than expected under a uniform distribution.

Empirically, $\omega(\theta)$ is estimated using a catalog of ``true" points ${\cal D}$ and a catalog of points drawn from a uniform random distribution ${\cal R}$. A simple estimator of the two-point correlation function is\footnote{An improved estimator was proposed by~\cite{1993ApJ...412...64L}, but we use the above for simplicity.}

\begin{equation}\label{eq:simple}
\omega(\theta) = \frac{DD(\theta)}{RR(\theta)}-1,
\end{equation}
where $DD(\theta)$ denotes the number of pairs in ${\cal D}$ separated by $\theta$, and $RR(\theta)$ denotes the corresponding number of pairs in ${\cal R}$, with both quantities normalized by the total number of possible pairs in each catalog:
\begin{equation}
    DD(\theta) = \frac{N_{pairs}^{\cal D}(\theta)}{N_D N_D}\text{~and~} RR(\theta) = \frac{N_{pairs}^{\cal R}(\theta)}{N_R N_R}.
\end{equation}

In practice, $\theta$ is discretized into a sequence of bins, and the pair counts are accumulated within each bin. Following \cite{Lapeer_2026}, we use logarithmically spaced intervals spanning from the resolution limit to the full field of view (FOV), while excluding bins for which insufficient data are available.

\subsection{Problem Statement} 
Our goal is to estimate the two-point correlation function for a set of target sources, given a larger set of sources that may be either targets or background. Formally, we are given $n$ points $S = \{\mathbf{s}_i\}_{i=1}^n$, each associated with an unknown label $y_i \in \{0, 1\}$ indicating whether it is a target (e.g., a star cluster) or a background source. Note that $RR(\theta)$ does not depend on the labels, so in what follows we focus on estimating $N_{\mathrm{pairs}}^{\cal D}(\theta)$, the number of pairs in which both points belong to the target class within each distance bin, in order to compute $\omega$ from Eq.~\ref{eq:simple}. Estimating the total number of targets $N_D$ is also required in principle, but this is straightforward compared to the pair counts.

In the active measurement setting, we assume access to a classifier $\hat{y_i} = h(\mathbf{s}_i)$ that provides noisy estimates of the probability that a point is a target. Our objective is to estimate $N_{\mathrm{pairs}}^{\cal D}(\theta)$ using the classifier's predictions, augmented with a small number of human-labeled examples.

\section{METHOD}
We frame the pairwise counting task as an edge-counting problem. Each source is represented as a vertex in a graph. For each distance bin, we draw an edge between two vertices whenever the distance between their corresponding sources falls within that bin. An edge is considered ``true" if both of its vertices are targets. Our analysis therefore focuses on counting the number of true edges within each distance bin. In the derivation that follows, we will focus on estimating the count in a single distance bin.

\subsection{Subset Estimator}
For each pair of vertices $(u,v)$ let
\begin{equation}
f(u,v)=
\begin{cases}
1 & \text{if edge $(u,v)$ is true}\\
0 & \text{otherwise.}
\end{cases}
\end{equation}
The total number of true edges is therefore
\begin{equation}
f(S) = \sum_{\substack{\{u,v\}\subseteq S\\ u\neq v}} f(u,v).
\end{equation}
A naive approach is to sample individual edges by selecting their corresponding pairs of vertices. However, this performs poorly because it ignores the contextual information that a labeled vertex shares with all its neighbors. Instead, our method samples entire subsets of vertices, using every edge within a subset to construct our estimate

Now, consider vertex-subsets $S_k \subseteq S$ of size $k$.  
We define the number of true edges inside a subset by
\begin{equation}
f(S_k) = \sum_{\substack{\{u,v\}\subseteq S_k\\ u\neq v}} f(u,v).
\end{equation}

Observe that each unordered pair $\{u,v\}\subseteq S$ appears in exactly $\binom{n-2}{k-2}$ distinct subsets of size $k$.  
Hence,
\begin{equation}
\sum_{\substack{S_k \subseteq S}} f(S_k) 
= \binom{n-2}{k-2} \sum_{\substack{\{u,v\}\subseteq S\\ u\neq v}} f(u,v) 
= \binom{n-2}{k-2} f(S).
\end{equation}

Therefore,
\begin{equation}
f(S) = \frac{1}{\binom{n-2}{k-2}} \sum_{\substack{S_k \subseteq S}} f(S_k).
\label{eq:total}
\end{equation}

\subsection{Monte Carlo Derivation}\label{est:mc}

Since there are $\binom{n}{k}$ vertex subsets of size $k$, we can rewrite this as an expectation with respect to the uniform distribution $U$ over subsets of size $k$:
\begin{align}
f(S) 
&= \frac{1}{\binom{n}{k}} \sum_{\substack{S_k \subseteq S}} 
  \frac{\binom{n}{k}}{\binom{n-2}{k-2}} f(S_k) \\
&= \mathbb{E}_{S_k \sim U} \left[ \frac{\binom{n}{k}}{\binom{n-2}{k-2}} f(S_k) \right].
\end{align}

Simplifying,
\begin{equation}
f(S) = \mathbb{E}_{S_k \sim U} \left[ \frac{n(n-1)}{k(k-1)} \, f(S_k) \right].
\end{equation}

This gives the subset Monte Carlo estimator:
\begin{equation}
\hat{f}_{\text{MC}}(S) = \frac{n(n-1)}{k(k-1)} f(S_k), \quad S_k \sim U.
\end{equation}

An additional observation we can make is that we can calculate the mean count over subsets if we know the mean count over pairs. By the definition of $f(S)$,

\begin{align}
f(S) &= \sum_{\substack{\{u,v\}\subseteq S \\ u\neq v}} f(u,v) \notag \\
&= \frac{n(n-1)}{2}\,\mathbb{E}_{u,v}[f(u,v)] \notag \\
&= \mathbb{E}_{S_k \sim U}\!\left[\frac{n(n-1)}{k(k-1)}\,f(S_k)\right]. \label{eq:subset-mean}
\end{align}

Thus we have the following relationship,

\begin{equation}
\frac{k(k-1)}{2}\,\mathbb{E}_{u,v}[f(u,v)] = 
\mathbb{E}_{S_k \sim U}[f(S_k)]. \label{eq:mean-relationship}
\end{equation}

\subsection{Importance Sampling Derivation}
Since we have access to a classifier, we can reduce the variance of our estimator by using importance sampling. Let $q(S_k)$ be our proposal distribution over k-subsets. Assuming we can draw samples from this distribution, then we can rewrite the total count as,

\begin{equation}
f(S) = \mathbb{E}_{S_k \sim q} \left[ \frac{f(S_k)}{\binom{n-2}{k-2} \, q(S_k)} \right].
\end{equation}

The optimal proposal distribution is proportional to the ground truth $f(S_k)$, as this would make the above expectation a constant, leading to zero variance. Thus we will define $q(S_k)$ to be proportional to the sum over predicted counts for each pair of vertices in $S_k$. Let $g(u, v) \ge 0$ be a predicted score for the unordered pair $\{u, v\}$ (e.g., a probability that edge $(u, v)$ is true). In our setting, we take $g(u,v) = h(u)h(v)$ to be the product of the predicted probabilities that u and v are true. We can now calculate $q(S_k)$ explicitly in terms of predicted counts. Define
\begin{equation}
G = \sum_{\substack{\{u,v\}\subseteq S\\ u\neq v}} g(u,v), 
\qquad
g(S_k) = \sum_{\substack{\{u,v\}\subseteq S_k\\ u\neq v}} g(u,v).
\end{equation}
Then
\begin{equation}
q(S_k) = \frac{g(S_k)}{\binom{n-2}{k-2}\, G}.
\end{equation}
Where the normalization term in the denominator comes from rearranging Eq.~\ref{eq:total}. The subset importance sampling estimator then simplifies to
\begin{equation}
\hat{f}_{\text{IS}}(S) = \frac{G\, f(S_k)}{g(S_k)}, 
\quad S_k \sim q.
\end{equation}

\subsection{Sampling Procedure}
We have now defined the main estimator for our method, however our definition depends on the ability to draw samples from some proposal $q(S_k)$ which is proportional to the predicted counts. The number of possible subsets of size $k$ quickly becomes intractable with even moderately sized $k$, so it is not feasible to realize $q(S_k)$ over all subsets. Inspired by Midzuno-Sen's method \citep[]{midzuno1952sampling, sen1952present}, we propose the following efficient sampling procedure:

\begin{enumerate}
    \item Define an initial distribution $q_0$ over pairs of vertices, where $q_0(u,v) \propto g(u,v)$
    \item Sample the first two vertices as an unordered pair $\{v_1,v_2\}$ from $q_0$.
    \item Sample the remaining $k-2$ vertices uniformly without replacement.
\end{enumerate}

\begin{proposition}
The probability of sampling a vertex subset $S_k$ under this procedure is proportional to
\begin{equation}
q(S_k) \propto \sum_{\substack{\{u,v\}\subseteq S_k\\ u\neq v}} g(u,v)
\end{equation}
which is the predicted count of true edges in the subset.
\end{proposition}

\begin{proof}
Fix a subset $S_k \subseteq S$ of size $k$. The probability of sampling this subset is the sum over probabilities of all possible permutations. 
This corresponds to choosing any unordered pair $\{u,v\}\subseteq S_k$ first (with probability $q_0(u,v)$), then selecting the remaining $k-2$ vertices from the $\binom{n-2}{k-2}$ equally likely $(k-2)$-subsets of $S\setminus\{u,v\}$. Thus
\begin{equation}
\begin{split}
q(S_k) &= \sum_{\substack{\{u,v\}\subseteq S_k}} q_0(u,v)\cdot \frac{1}{\binom{n-2}{k-2}} \\
&= \frac{1}{\binom{n-2}{k-2}} \sum_{\substack{\{u,v\}\subseteq S_k}} q_0(u,v),
\end{split}
\end{equation}
which is proportional to $\sum_{\{u,v\}\subseteq S_k} g(u,v)$ since $q_0(u,v)\propto g(u,v)$.
\end{proof}
\vspace{-2pt}
We can now use this procedure to efficiently sample our proposal distribution, without having to enumerate over all possible $k$-subsets. Furthermore, this process is self-similar. Given an existing subset of size $k$, adding 1 uniformly sampled vertex provides a $k+1$-subset drawn from the $k+1$ proposal distribution. Rather than drawing multiple independent subsets, we construct a single progressively growing subset, adding vertices one at a time and updating the estimate at each step.

\subsection{Multiple Bins}
For the two-point correlation problem, we must estimate pair counts across all distance bins simultaneously. Each bin shares the same set of vertices but contains a disjoint set of edges. Running separate estimators for each bin is inefficient, as it would require labeling many more vertices and would fail to exploit the overlap across bins.

Our approach leverages the fact that only the initial vertex pair must be sampled separately for each bin. We draw one pair per bin, with probability proportional to its predicted count, and then sample additional vertices uniformly at random, updating all bins simultaneously.

A complication arises when sampling without replacement: if the initial pairs differ across bins, the corresponding remaining candidate sets also differ. We address this by sampling from the full vertex set while ignoring vertices that have already been used in a bin’s initial sample. If a sampled vertex duplicates one already selected within a bin, that draw is discarded; if it has already been selected in another bin, it is reused. As a result, the sampled subsets may differ slightly in size across bins. This completes our sampling algorithm, which is summarized in Algorithm~1.

\begin{algorithm}[htb!]
\caption{Multi-bin Pairwise Estimation}
\KwIn{Set of vertices $S$, predicted edge scores $\{g_b(u,v)\}$ for each bin $b = 1,\dots,B$}
\KwOut{For each bin $b$: sampled subset $S^b_n$ and sequence of estimates $\{\hat{f}^b_{\text{IS}}(S^b_i)\}$}

\BlankLine
\ForEach{bin $b$}{
    $G_b \gets \sum_{\{u,v\} \subseteq S} g_b(u,v)$\;
}

\ForEach{bin $b$}{
    $q_{0,b}(u,v) \propto g_b(u,v)$ for all unordered pairs $\{u,v\}$\;
}

\ForEach{bin $b$}{
    Sample $\{v^b_1, v^b_2\} \sim q_{0,b}$\;
    
    $S^b_2 \gets \{v^b_1, v^b_2\}$\;
    
    Label $v^b_1, v^b_2$ to compute $f_b(S^b_2)$\;
    
    $g_b(S^b_2) \gets g_b(v^b_1, v^b_2)$\;
    
    $\hat{f}^b_{\text{IS}}(S^b_2) \gets G_b \cdot \frac{f_b(S^b_2)}{g_b(S^b_2)}$\;
    
    $\text{Estimates}_b \gets [\hat{f}^b_{\text{IS}}(S^b_2)]$\;
}

$n \gets |S|$\;

$U \gets S$\; 

\For{$t = 3$ \KwTo $n$}{
    Sample $v \sim \text{Uniform}(U)$\;
    
    $U \gets U \setminus \{v\}$\;
    
    \ForEach{bin $b$}{
        \eIf{$v \notin S^b_{t-1}$}{
            $S^b_t \gets S^b_{t-1} \cup \{v\}$\;
            
            Label $v$ to compute $f_b(S^b_t)$\;
            
            $g_b(S^b_t) \gets \sum_{\{u,v\} \subseteq S^b_t} g_b(u,v)$\;
            
            $\hat{f}^b_{\text{IS}}(S^b_t) \gets G_b \cdot \frac{f_b(S^b_t)}{g_b(S^b_t)}$\;
            
            Append $\hat{f}^b_{\text{IS}}(S^b_t)$ to $\text{Estimates}_b$\;
        }{
            $S^b_t \gets S^b_{t-1}$\tcp*{No update}
        }
    }
}

\Return{$\{(S^b_n, \text{Estimates}_b)\}_{b=1}^B$}
\end{algorithm}

\section{VARIANCE ESTIMATION}
We first present an exact expression for the variance of the Monte Carlo subset estimator and introduce a procedure for computing an unbiased variance estimate from a single $k$-subset. We then consider the more challenging problem of estimating the variance of the importance sampling estimator. Our proposed approach yields an unbiased estimate of a Taylor series approximation to the true variance. In later sections, we present empirical results showing that this approximation closely matches the true variance.

\subsection{Variance of $\hat{f}_{\text{MC}}$}  
In Eq.~\ref{eq:mean-relationship} we showed that the mean of $f(S_k)$ over all $k$-subsets can be expressed in terms of the mean of $f(u,v)$ over distinct vertex pairs. A similar reduction holds for the variance. It can be written in terms of expectations involving $f(u,v)f(w,z)$ for pairs of distinct vertex pairs $\{u,v\},\{w,z\}$.  

\begin{proposition}  
The variance of $\hat{f}_{\normalfont \text{MC}}$ is given by  
\begin{equation}
\boxed{%
\begin{aligned}
\operatorname{Var}[\hat{f}_{\text{MC}}(S)]
&= \frac{n(n-1)(n-k)}{2k(k-1)}
   \Bigl[ 2(E_{2}-E_{3})kn \\
&\quad + (E_{1}-4E_{2}+3E_{3})(n+k-1) \Bigr]
\end{aligned}}
\label{eq:mc-var}
\end{equation}
\label{prop:mc-var}
\end{proposition}  

The terms $E_1$, $E_2$, and $E_3$ are expectations of $f(u,v)f(w,z)$ over different distributions of vertex pairs.  

\begin{equation}  
\begin{aligned}  
E_2 &= \mathbb{E}_{\{u,v\}}\bigl[f(u,v)^{2}\bigr], \\  
E_3 &= \mathbb{E}_{|\{u,v,w\}|=3}\bigl[f(u,v)f(v,w)\bigr], \\  
E_4 &= \mathbb{E}_{|\{u,v,w,z\}|=4}\bigl[f(u,v)f(w,z)\bigr].  
\end{aligned}  
\end{equation}  

The expectation $E_2$ is defined by drawing the unordered pair $(u,v)$ uniformly from the $\binom{n}{2}$ pairs of distinct vertices. The expectation $E_3$ is defined by drawing $(u,v,w)$ uniformly from all ordered triples of distinct vertices and then forming the two pairs $(u,v)$ and $(w,v)$ so that the pairs share exactly one vertex. The expectation $E_4$ is defined by drawing $(u,v,w,z)$ uniformly from all ordered quadruples of distinct vertices and then forming the two pairs $(u,v)$ and $(w,z)$ so that the pairs do not overlap.  

In summary $E_2$ corresponds to identical pairs of vertices, $E_3$ corresponds to pairs that overlap at exactly one vertex, and $E_4$ corresponds to pairs that are disjoint.  

The full proof of Proposition \ref{prop:mc-var} can be found in the Supplementary Material. Importantly, these expectations depend only on vertex pairs rather than subsets, which makes them considerably more tractable. Additionally, this expression is linear in $E_2$, $E_3$, and $E_4$, allowing us to form an unbiased estimator of the variance by estimating these expectations individually.

We can estimate each $E_i$ using sample means. Specifically, our estimate is obtained by evaluating all pair–of–pair combinations within the sampled subset $S_k$,
\begin{equation}
    \hat{E_i} = \frac{1}{n_i}\sum_{\substack{\{u,v,w,z\} \subseteq S_k \\
    u \neq v, w \neq z\\
    |\{u,v,w,z\}| = i}} f(u,v)f(w,z), 
\quad S_k \sim U.
\end{equation}
where $n_i$ is the number of terms in the sum. 
\begin{proposition}
Our sample means $\hat{E}_i$ are unbiased estimates of the population means $E_i$
\end{proposition}

We can prove this by directly computing the expectation using indicator variables and the exact expressions for $n_i$. The details are provided in the Supplementary Material.
After obtaining estimates $\hat{E_2}$, $\hat{E_3}$, and $\hat{E_4}$ with vertices in $S_k$, we can then plug the result into Equation \ref{eq:mc-var} to get our unbiased estimate of the variance.

\subsection{Variance of $\hat{f}_{\text{IS}}$}
For our main method which incorporates importance sampling, we cannot derive a linear expression for the variance because we have a ratio of random variables. However, we can derive a linear expression for a close approximation of the variance. We can start by writing the variance as,
\begin{equation}
    \operatorname{Var}_q[\hat{f}_{\text{IS}}(S)] = G^2\left(\mathbb{E}_q\left[\frac{f^2(S_k)}{g^2(S_k)}\right] - \mathbb{E}_q\left[\frac{f(S_k)}{g(S_k)}\right]^2\right)
\end{equation}

Converting to a uniform expectation,
\begin{equation}
   = G\frac{n(n-1)}{k(k-1)}\mathbb{E}\left[\frac{f^2(S_k)}{g(S_k)}\right] - \left(\frac{n(n-1)}{k(k-1)}\right)^2\mathbb{E}\left[f(S_k)\right]^2
\end{equation}

We can approximate the first expectation using the delta method \citep[e.g.,][Section~5.3.1]{bickel-doksum}. Viewing this ratio as a multivariate function $h(X,Y) = X/Y$, we can form an approximation using the second order Taylor series about the point $(\mathbb{E}[f^2(S_k)], \mathbb{E}[g(S_k)])$. 
As $k$ increases, we expect $X$ and $Y$ to concentrate around their means, so this approximation becomes increasingly accurate. This gives,
\begin{equation}
\begin{aligned}
\mathbb{E}\!\left[\frac{f^2(S_k)}{g(S_k)}\right]
&\approx \frac{\mathbb{E}[f^2(S_k)]}{\mathbb{E}[g(S_k)]} 
   - \frac{\operatorname{Cov}\!\bigl(f^2(S_k),g(S_k)\bigr)}{\mathbb{E}[g(S_k)]^{2}} \\
&\quad + \frac{\mathbb{E}[f^2(S_k)]\,\operatorname{Var}[g(S_k)]}{\mathbb{E}[g(S_k)]^{3}}.
\end{aligned}
\end{equation}
We can further decompose,
\begin{equation}
\scalebox{0.95}{$
\operatorname{Cov}(f^2(S_k),g(S_k))
= \mathbb{E}[f^2(S_k)g(S_k)]
-\mathbb{E}[f^2(S_k)]\mathbb{E}[g(S_k)]
$}
\end{equation}

Since $g(u,v)$ is known for all pairs, we can compute $\mathbb{E}[g(S_k)]$ and $\operatorname{Var}[g(S_k)]$ exactly. Combining these last 3 equations, our expression for the variance is a linear combination of $\mathbb{E}[f^2(S_k)]$, $\mathbb{E}[f^2(S_k)g(S_k)]$, and $\mathbb{E}[f(S_k)]^2$. 

Following a similar approach to the previous section, we can express $\mathbb{E}[f^2(S_k)]$ and $\mathbb{E}[f(S_k)]^2$ as linear combinations of $E_1$, $E_2$, and $E_3$. The mixed term, $\mathbb{E}[f^2(S_k)g(S_k)]$, can also be decomposed, but in this case into a linear combination of expectations involving triplets of vertex pairs, $\mathbb{E}[f(u,v)f(w,z)g(r,s)]$, which we denote $D_2,\dots,D_6$. In summary, it suffices to construct an unbiased estimator for each $E_i$ and $D_i$, as their linear combination yields an unbiased estimator of our second order variance approximation. A complete derivation and explicit formulas are provided in the Supplementary Material.

Unlike the previous section, the $E_i$ and $D_i$ terms can no longer be estimated directly by a sample mean, since the vertices are not sampled uniformly. To correct for this mismatch between the uniform expectation and non-uniform samples, we employ the unbiased estimator of \cite{bbbe4e97-7833-3faa-820c-4c61f82fe965}, which reweights each term by the inverse of its inclusion probability. This leads to the estimator
\begin{equation}
    \hat{E}_i^{\text{IS}} = \frac{1}{N_i}\sum_{\substack{u,v,w,z \subseteq S_k \\
    u \neq v, w \neq z\\
    |\{u,v,w,z\}| = i}} \frac{f(u,v)f(w,z)}{\pi(u,v,w,z)}, 
\quad S_k \sim q,
\end{equation}
where $N_i$ is the number of such pairs in the full dataset $S$, and $\pi(u,v,w,z)$ is the probability that the vertices $u,v,w,z$ appear in a sampled $k$-subset. The estimators for $\hat{D}_i^{\text{IS}}$ are defined similarly.

This completes the derivation of all quantities needed for our variance approximation. The corresponding procedure is summarized in Algorithm \ref{alg:var}. With additional algebraic simplifications, all $\hat{E}_i$ and $\hat{D}_i$ can be computed in $O(E^2)$ time, where $E$ is the number of true edges in the sampled $k$-subset.

\begin{algorithm}[htb!]
\label{alg:var}
\caption{Variance Estimation}
\KwIn{Subset of vertices $S_k$, predicted edge scores $g(u,v)$, ground truth labels $f(u,v)$ in $S_k$}
\KwOut{Variance Estimate $\hat{V_k}$}

Compute $\mathbb{E}[g(S_k)]$ and $\operatorname{Var}[g(S_k)]$\;

Estimate $\hat{E_2}$, $\hat{E_3}$, and $\hat{E_4}$ using Horvitz–Thompson on pairs of pairs in $S_k$.\;

Compute unbiased estimate for $\mathbb{E}[f^2(S_k)]$ and $\mathbb{E}[f(S_k)]^2$ using $\hat{E}_2$, $\hat{E_3}$, and $\hat{E_4}$\;

Estimate $\hat{D_2}$, ..., $\hat{D_6}$ using Horvitz–Thompson on triplets of pairs in $S_k$.\;

Compute unbiased estimate for $\mathbb{E}[f^2(S_k)g(S_k)]$  using $\hat{D_2}$, ...,$\hat{D_6}$\;

Calculate $\hat{V}_k$ from $\mathbb{E}[f^2(S_k)]$, $\mathbb{E}[f(S_k)]^2$, and $\mathbb{E}[f^2(S_k)g(S_k)]$.\;

\Return{$\hat{V}_k$}
\end{algorithm}

\section{DATASETS AND EXPERIMENTS}

\subsection{JWST FEAST Data}
For our two-point correlation analysis, we use imaging data from the James Webb Space Telescope (JWST). The dataset comprises multi-band observations, with each band corresponding to a distinct wavelength. Star clusters are gravitationally bound groups of stars that formed from the same molecular cloud and are therefore important tracers of recent star formation. Source labels come from the JWST–FEAST program (Feedback in Emerging extrAgalactic Star clusTers), which studies star formation and stellar feedback in nearby galaxies. Specifically, we use data from the Cycle 1 General Observer program 1783 (JWST-GO-1783; P.I.: A. Adamo). The FEAST catalogs provide precise positions and classifications of young star clusters (YSCs), serving as ground-truth annotations. In our experiments, we use this ground truth to simulate a human-in-the-loop annotator, though in practice our method is designed for datasets without such catalogs.

Each galaxy is represented by a large multi-channel image and a catalog of YSCs containing pixel coordinates, cluster class, and other metadata. YSCs fall into three groups: two early types (eYSCI and eYSCII), which are detectable in the infrared with JWST and have been identified as young, dust-embedded star clusters \citep[]{2024ApJ...971..115G, 2025arXiv250901670P}, and optical YSCs (oYSCs), observed with the Hubble Space Telescope (HST), which are younger than 10 million years old and have already emerged from their natal clouds of dust and gas. In this work, we treat eYSCI and eYSCII clusters as target sources and all others as background.

\subsection{NGC628 and NGC4449 Galaxies}
We analyze two galaxies at different scales: NGC628, a nearby spiral galaxy, and NGC4449, a dwarf galaxy.

The NGC628 dataset contains 2,618 sources, of which 885 are classified as eYSCI or eYSCII. The two-point correlation function is computed across 14 logarithmically spaced distance bins. This yields approximately 1.6 million candidate edges, of which 177,085 are true edges.

The NGC4449 dataset contains 659 sources, of which 260 are classified as eYSCI or eYSCII. Here, the two-point correlation function is computed across 13 logarithmically spaced distance bins, yielding 216,798 candidate edges, of which 33,664 are true edges. 

\subsection{Experimental Protocol}

\begin{figure*}[t!]
    \centering
    \includegraphics[width=0.4\textwidth]{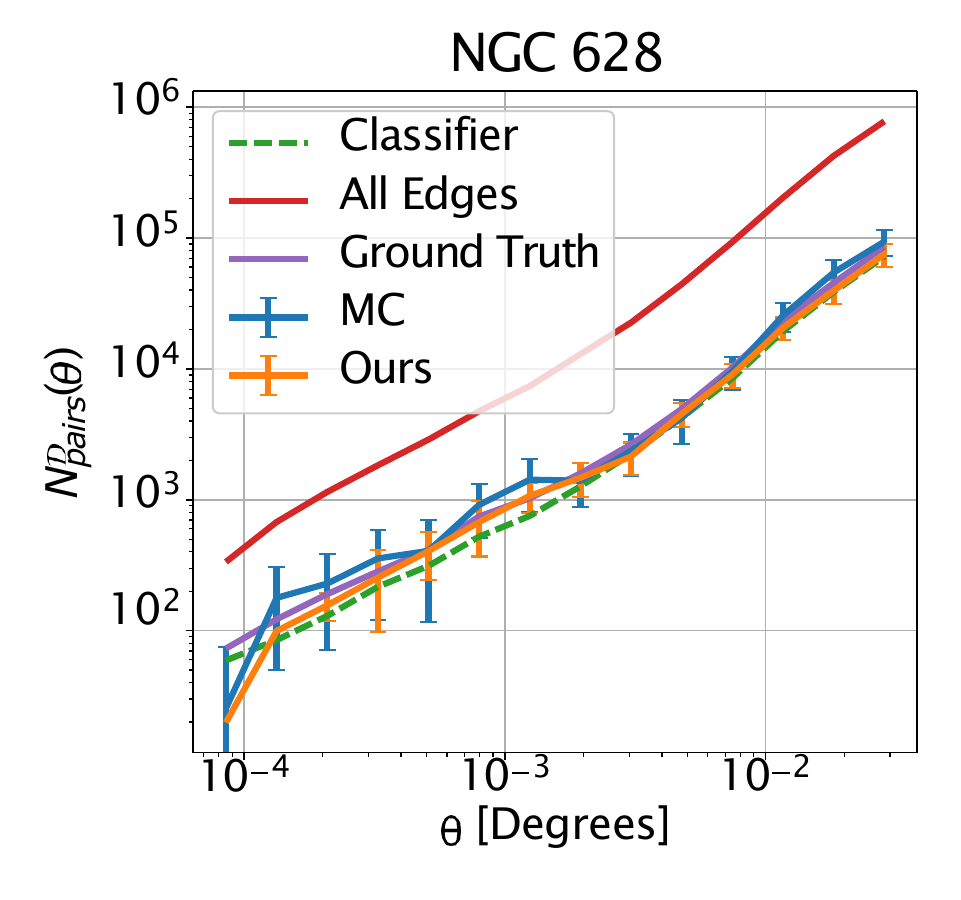}
    \hspace{0.05\textwidth} %
    \includegraphics[width=0.4\textwidth]{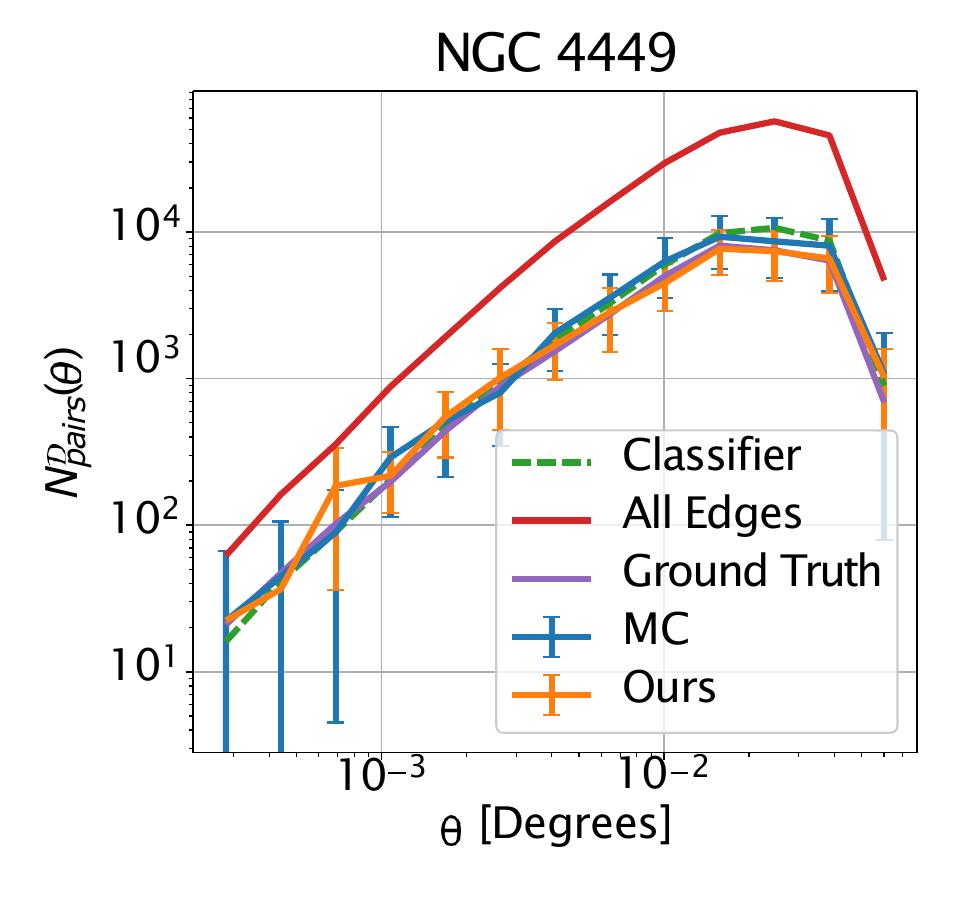}
    \caption{Predicted counts and estimated 95\% confidence intervals of a single trial 
    with 20\% of sources labeled. Our method has less error and smaller error bars compared 
    to MC. There are significantly fewer ground truth edges compared to all edges, illustrating the importance of source classification.}
    \label{fig:main-left}

    \vspace{1.5em} %

    \includegraphics[width=0.4\textwidth]{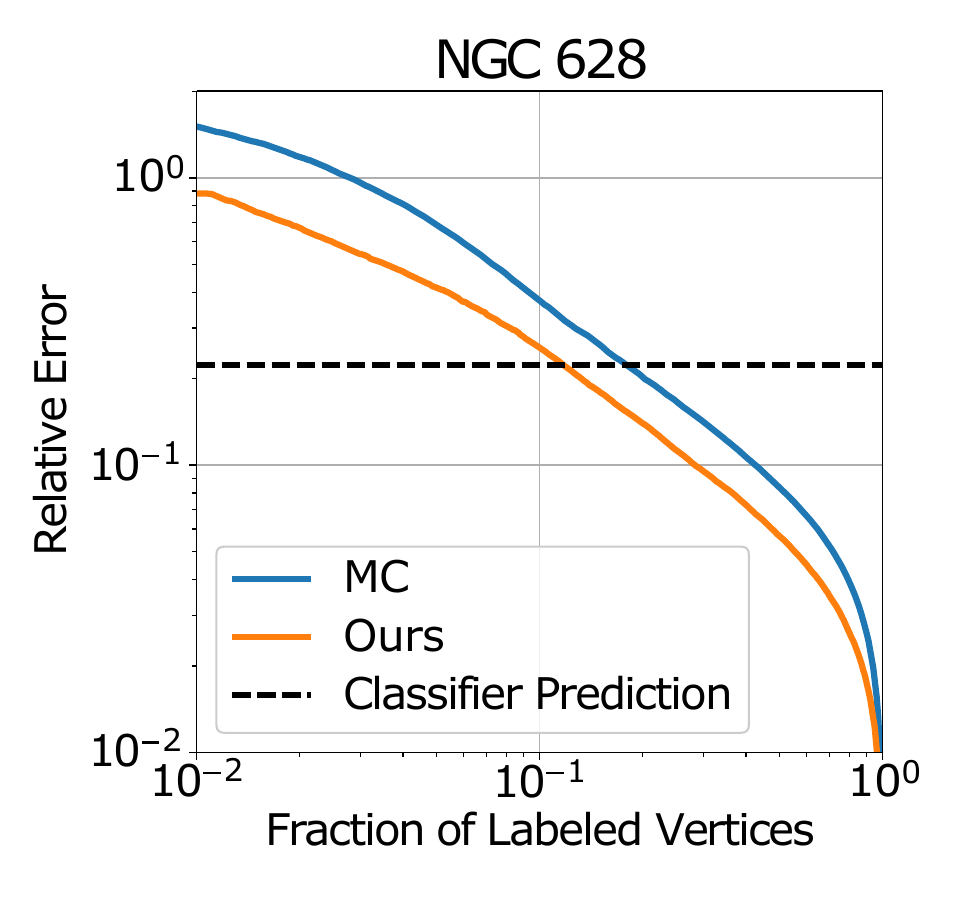}
    \hspace{0.05\textwidth}
    \includegraphics[width=0.4\textwidth]{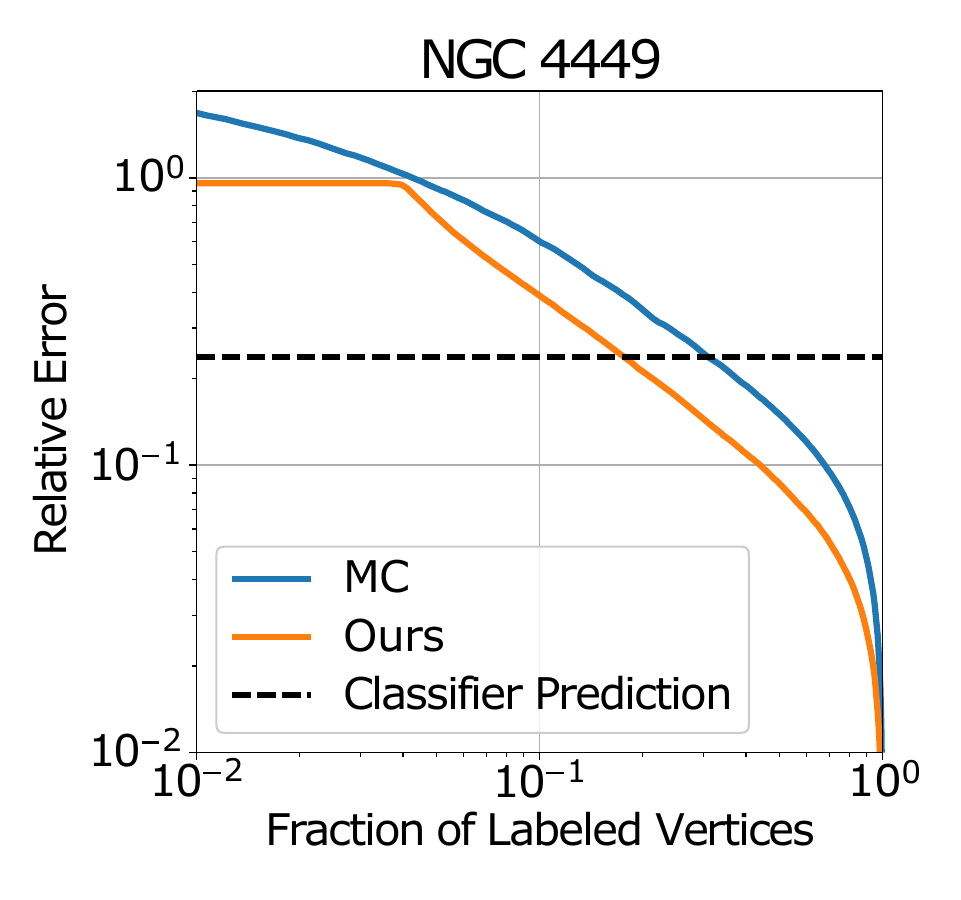}
    \caption{Relative error of the estimated counts for the two galaxies as a function 
    of labeled vertices. The errors are averaged over 2000 trials and all bins. Our method consistently reduces error over MC, and beats the classifier when around 20\% of the vertices are labeled.}
    \label{fig:main-right}
\end{figure*}

\begin{figure}[t!]
    \centering
    \includegraphics[width=0.99\linewidth]{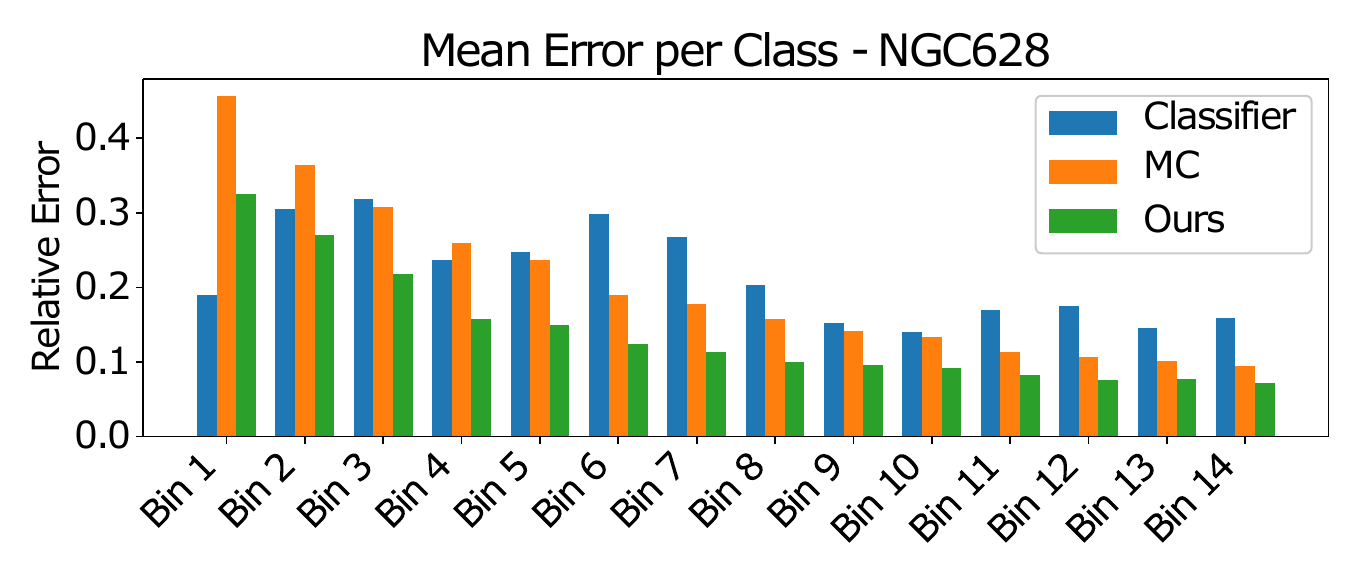}
    \includegraphics[width=0.99\linewidth]{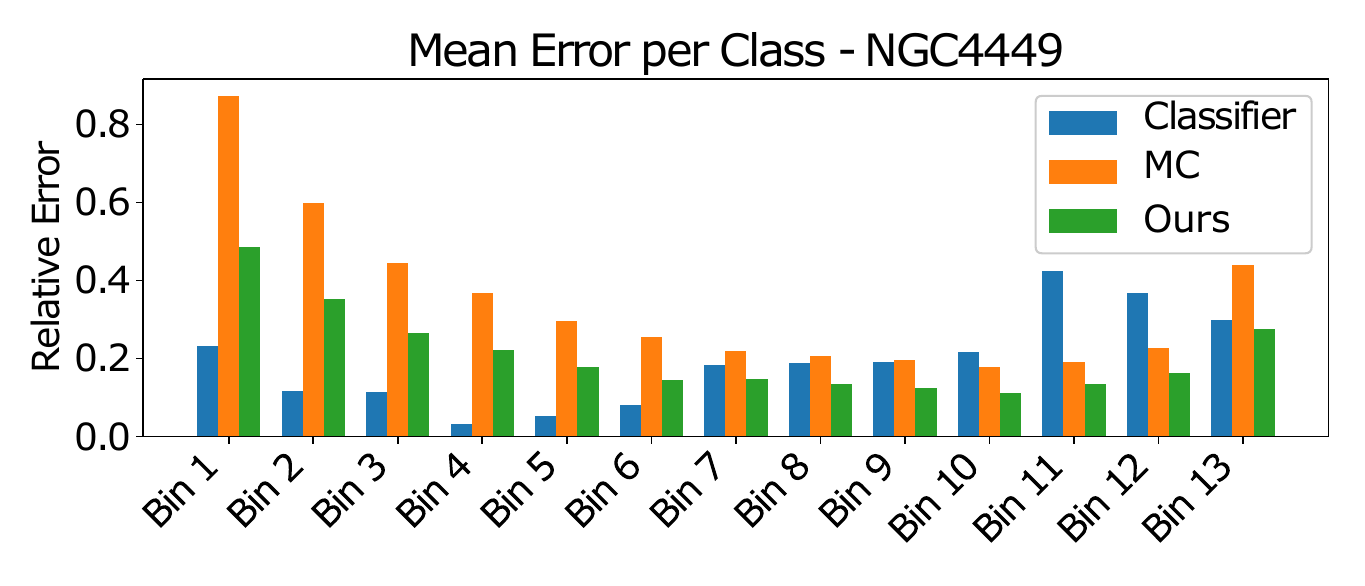}
    \caption{Errors across bins with 20\% labeled vertices. Our method consistently outperforms the Monte Carlo baseline, with the lowest error on the largest bins.}
    \label{fig:confidence}
\end{figure}

\begin{figure}[t!]
    \centering
    \includegraphics[width=1.02\linewidth]{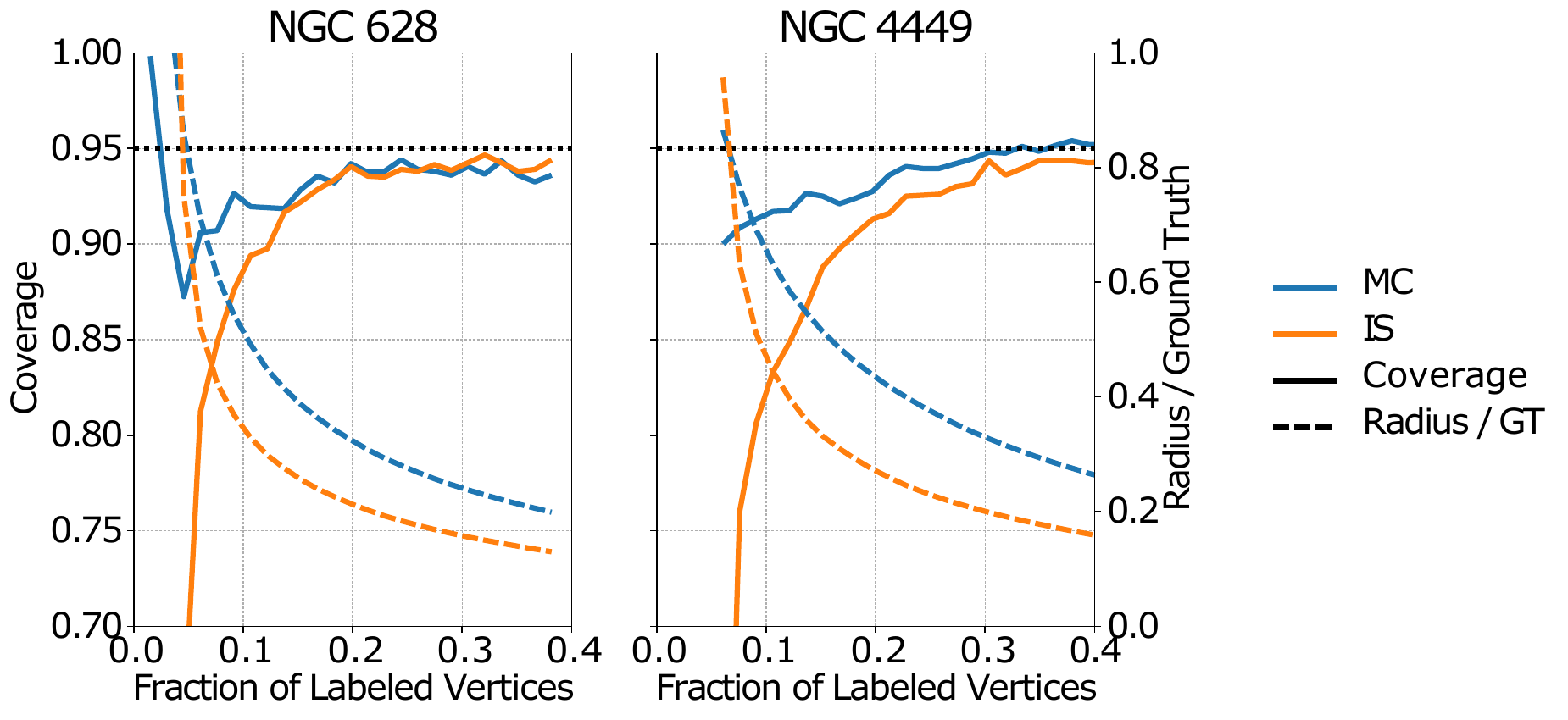}
    \caption{Evaluation of estimated 95\% confidence intervals for both galaxies (Bin 10): coverage and radius relative to the ground truth count, plotted against the number of labeled vertices (averaged over 2000 trials).}
    \label{fig:coverage}
\end{figure}

\paragraph{Classifier}
We use both real and simulated classifier probabilities in our experiments.

For NGC628, the real predictions are generated using StarcNet~\citep{pmcmjas_apj2021}, a neural network designed to identify star clusters. We train a new model on JWST observations of NGC628 and obtain predicted class probabilities through $k$-fold cross-validation, training on nine folds and predicting on the remaining fold. Repeating this process 10 times yields held-out predictions for all sources. The resulting classifier achieves an accuracy of approximately 85\%. Full training details are provided in the Supplementary Material.

For NGC4449, we instead simulate predictions by sampling from beta distributions. Specifically, we use $\alpha=3, \beta=1$ for positive samples and $\alpha=1, \beta=3$ for negative samples, producing probabilities corresponding to a classifier with an accuracy of about 87.5\%.

\paragraph{Baselines}
We evaluate our method against two baselines. First, a Monte Carlo (MC) estimator (Section~\ref{est:mc}), which samples sources uniformly at random. This baseline highlights the gains achieved by incorporating the classifier. Second, a direct classifier-based estimator. We use the total predicted count, $g(S)$, as our estimate for each distance bin. Unlike the other methods, this baseline can have bias and does not provide error bars.

\paragraph{Evaluation Metrics}
We use the fractional error as our evaluation metric, averaged across all trials: 
$\frac{1}{M} \sum_{m=1}^M \frac{|\hat{f}(S) - f(S)|}{f(S)}$, 
where $\hat{f}(S)$ is our estimate, $f(S)$ is the ground truth count, and $M$ is the number of trials. In some cases, we also average this fractional error across all bins.

\section{RESULTS}

\paragraph{Main Results} 
Figure~\ref{fig:main-left} illustrates the predicted pair counts and estimated confidence intervals from a single trial for both NGC 628 and NGC 4449, showing that our method produces accurate estimates with well-calibrated error bars. Figure~\ref{fig:main-right} reports the fractional error of the predicted counts as a function of the fraction of labeled vertices, averaged over 500 trials and all bins. Across all settings, our method consistently outperforms the MC baseline, achieving a 33\% reduction in relative error and requiring 33\% fewer labels to reach 10\% relative error.

Figure~\ref{fig:confidence} compares the relative error across all bins when 20\% of the vertices are labeled. Across all bins, our method yields lower average error than the MC baseline, and it outperforms the classifier prediction for most bins.

\paragraph{Confidence Intervals}
Figure~\ref{fig:coverage} evaluates the estimated confidence intervals for Bin 10, which we selected because of its moderate size. With sufficient labeled data, the coverage of our 95\% confidence intervals approaches the ideal 95\% target. Our intervals are also narrower than those from the MC approach, which indicates more precise uncertainty estimates. Performance depends on bin size, with smaller bins showing worse coverage and larger bins showing better coverage. Results for all bins are included in the Supplementary Material.

\section{CONCLUSION}

We present a human-in-the-loop framework for efficiently estimating two-point correlation functions (2PCFs) of young star clusters. Using a pre-trained classifier to guide sampling, our method adaptively selects informative points for labeling, yielding unbiased pair-count estimates and confidence intervals across distance bins. Compared to standard Monte Carlo approaches, it achieves lower variance and substantially reduces annotation effort, enabling scalable and statistically rigorous 2PCF measurement in large astronomical catalogs.

While our experiments focus on astronomy, the two-point correlation function is used across science and engineering beyond cosmology, including in ecological modeling of the spatial distribution of trees \citep{SEKRETENKO1998113} and birds \citep{sun2014reduced}, as well as in materials science \citep{PhysRevLett.77.2581}. More generally, our framework applies to any problem that can be reduced to counting weighted edges in a graph, where the weight of an edge depends on the ground-truth labels of its two endpoints. For example, in social media, it could be used to study how likely users sharing the same attribute are to be connected.

The method may be less effective when the classifier is inaccurate or the target subset is extremely sparse. Future work could jointly adapt the sampling strategy and improve the classifier as new labels are collected.

\subsubsection*{Acknowledgements}
We thank Daniela Calzetti and Drew Lapeer for sharing their astronomical domain expertise, providing guidance on the dataset, and offering valuable feedback on the manuscript. We also thank Gustavo Perez for his guidance with the cluster prediction, and Jinlin Lai for his helpful feedback on the manuscript. This work was supported in part by National Science Foundation grants \#2504073 and \#2406687. 

This work is based on observations made with the NASA/ESA/CSA James Webb Space Telescope, which is operated by the Association of Universities 
for Research in Astronomy, Inc., under NASA contract NAS 5-03127. 
These observations are associated with program \# 1783. Support for program \# 1783 was provided by NASA through a grant from the Space Telescope Science Institute, which is operated by the Association of Universities for Research in Astronomy, Inc., under NASA contract NAS 5-03127.

\bibliography{references.bib}

\begin{thebibliography}{}

\bibitem[{Adamo} et~al., 2017]{2017ApJ...841..131A}
{Adamo}, A., {Ryon}, J.~E., {Messa}, M., {Kim}, H., {Grasha}, K., {Cook}, D.~O., {Calzetti}, D., {Lee}, J.~C., {Whitmore}, B.~C., {Elmegreen}, B.~G., {Ubeda}, L., {Smith}, L.~J., {Bright}, S.~N., {Runnholm}, A., {Andrews}, J.~E., {Fumagalli}, M., {Gouliermis}, D.~A., {Kahre}, L., {Nair}, P., {Thilker}, D., {Walterbos}, R., {Wofford}, A., {Aloisi}, A., {Ashworth}, G., {Brown}, T.~M., {Chandar}, R., {Christian}, C., {Cignoni}, M., {Clayton}, G.~C., {Dale}, D.~A., {de Mink}, S.~E., {Dobbs}, C., {Elmegreen}, D.~M., {Evans}, A.~S., {Gallagher}, III, J.~S., {Grebel}, E.~K., {Herrero}, A., {Hunter}, D.~A., {Johnson}, K.~E., {Kennicutt}, R.~C., {Krumholz}, M.~R., {Lennon}, D., {Levay}, K., {Martin}, C., {Nota}, A., {{\"O}stlin}, G., {Pellerin}, A., {Prieto}, J., {Regan}, M.~W., {Sabbi}, E., {Sacchi}, E., {Schaerer}, D., {Schiminovich}, D., {Shabani}, F., {Tosi}, M., {Van Dyk}, S.~D., and {Zackrisson}, E. (2017).
\newblock {Legacy ExtraGalactic UV Survey with The Hubble Space Telescope: Stellar Cluster Catalogs and First Insights Into Cluster Formation and Evolution in NGC 628}.
\newblock {\em The Astrophysical Journal}, 841(2):131.

\bibitem[Angelopoulos et~al., 2023]{angelopoulos2023prediction}
Angelopoulos, A.~N., Bates, S., Fannjiang, C., Jordan, M.~I., and Zrnic, T. (2023).
\newblock Prediction-powered inference.
\newblock {\em Science}, 382(6671):669--674.

\bibitem[Banerji et~al., 2010]{banerji2010galaxy}
Banerji, M., Lahav, O., Lintott, C.~J., Abdalla, F.~B., Schawinski, K., Bamford, S.~P., Andreescu, D., Murray, P., Raddick, M.~J., Slosar, A., et~al. (2010).
\newblock Galaxy zoo: reproducing galaxy morphologies via machine learning.
\newblock {\em Monthly Notices of the Royal Astronomical Society}, 406(1):342--353.

\bibitem[Bickel and Doksum, 2015]{bickel-doksum}
Bickel, P.~J. and Doksum, K.~A. (2015).
\newblock {\em Mathematical Statistics: Basic Ideas and Selected Topics}, volume~I.
\newblock Chapman \& Hall, 2nd edition.

\bibitem[{Calzetti} et~al., 2015]{2015AJ....149...51C}
{Calzetti}, D., {Lee}, J.~C., {Sabbi}, E., {Adamo}, A., {Smith}, L.~J., {Andrews}, J.~E., {Ubeda}, L., {Bright}, S.~N., {Thilker}, D., {Aloisi}, A., {Brown}, T.~M., {Chandar}, R., {Christian}, C., {Cignoni}, M., {Clayton}, G.~C., {da Silva}, R., {de Mink}, S.~E., {Dobbs}, C., {Elmegreen}, B.~G., {Elmegreen}, D.~M., {Evans}, A.~S., {Fumagalli}, M., {Gallagher}, III, J.~S., {Gouliermis}, D.~A., {Grebel}, E.~K., {Herrero}, A., {Hunter}, D.~A., {Johnson}, K.~E., {Kennicutt}, R.~C., {Kim}, H., {Krumholz}, M.~R., {Lennon}, D., {Levay}, K., {Martin}, C., {Nair}, P., {Nota}, A., {{\"O}stlin}, G., {Pellerin}, A., {Prieto}, J., {Regan}, M.~W., {Ryon}, J.~E., {Schaerer}, D., {Schiminovich}, D., {Tosi}, M., {Van Dyk}, S.~D., {Walterbos}, R., {Whitmore}, B.~C., and {Wofford}, A. (2015).
\newblock {Legacy Extragalactic UV Survey (LEGUS) With the Hubble Space Telescope. I. Survey Description}.
\newblock {\em The Astronomical Journal}, 149(2):51.

\bibitem[Citovsky et~al., 2023]{citovsky2023leveraging}
Citovsky, G., DeSalvo, G., Kumar, S., Ramalingam, S., Rostamizadeh, A., and Wang, Y. (2023).
\newblock Leveraging importance weights in subset selection.
\newblock In {\em The Eleventh International Conference on Learning Representations}.

\bibitem[{Grasha} et~al., 2017a]{2017ApJ...840..113G}
{Grasha}, K., {Calzetti}, D., {Adamo}, A., {Kim}, H., {Elmegreen}, B.~G., {Gouliermis}, D.~A., {Dale}, D.~A., {Fumagalli}, M., {Grebel}, E.~K., {Johnson}, K.~E., {Kahre}, L., {Kennicutt}, R.~C., {Messa}, M., {Pellerin}, A., {Ryon}, J.~E., {Smith}, L.~J., {Shabani}, F., {Thilker}, D., and {Ubeda}, L. (2017a).
\newblock {The Hierarchical Distribution of the Young Stellar Clusters in Six Local Star-forming Galaxies}.
\newblock {\em The Astrophysical Journal}, 840(2):113.

\bibitem[{Grasha} et~al., 2017b]{2017ApJ...842...25G}
{Grasha}, K., {Elmegreen}, B.~G., {Calzetti}, D., {Adamo}, A., {Aloisi}, A., {Bright}, S.~N., {Cook}, D.~O., {Dale}, D.~A., {Fumagalli}, M., {Gallagher}, III, J.~S., {Gouliermis}, D.~A., {Grebel}, E.~K., {Kahre}, L., {Kim}, H., {Krumholz}, M.~R., {Lee}, J.~C., {Messa}, M., {Ryon}, J.~E., and {Ubeda}, L. (2017b).
\newblock {Hierarchical Star Formation in Turbulent Media: Evidence from Young Star Clusters}.
\newblock {\em The Astrophysical Journal}, 842(1):25.

\bibitem[{Gregg} et~al., 2024]{2024ApJ...971..115G}
{Gregg}, B., {Calzetti}, D., {Adamo}, A., {Bajaj}, V., {Ryon}, J.~E., {Linden}, S.~T., {Correnti}, M., {Cignoni}, M., {Messa}, M., {Sabbi}, E., {Gallagher}, J.~S., {Grasha}, K., {Pedrini}, A., {Gutermuth}, R.~A., {Melinder}, J., {Kotulla}, R., {P{\'e}rez}, G., {Krumholz}, M.~R., {Bik}, A., {{\"O}stlin}, G., {Johnson}, K.~E., {Bortolini}, G., {Smith}, L.~J., {Tosi}, M., {Maji}, S., and {Faustino Vieira}, H. (2024).
\newblock {Feedback in Emerging Extragalactic Star Clusters, FEAST: The Relation between 3.3 {\ensuremath{\mu}}m Polycyclic Aromatic Hydrocarbon Emission and Star Formation Rate Traced by Ionized Gas in NGC 628}.
\newblock {\em The Astrophysical Journal}, 971(1):115.

\bibitem[Hamilton et~al., 2025]{hamilton2025active}
Hamilton, M., Lai, J., Zhao, W., Maji, S., and Sheldon, D. (2025).
\newblock Active measurement: Efficient estimation at scale.
\newblock In {\em Neural Information Processing Systems (NeurIPS)}.

\bibitem[Horvitz and Thompson, 1952]{bbbe4e97-7833-3faa-820c-4c61f82fe965}
Horvitz, D.~G. and Thompson, D.~J. (1952).
\newblock A generalization of sampling without replacement from a finite universe.
\newblock {\em Journal of the American Statistical Association}, 47(260):663--685.

\bibitem[{Landy} and {Szalay}, 1993]{1993ApJ...412...64L}
{Landy}, S.~D. and {Szalay}, A.~S. (1993).
\newblock {Bias and Variance of Angular Correlation Functions}.
\newblock {\em The Astrophysical Journal}, 412:64.

\bibitem[Lapeer et~al., 2026]{Lapeer_2026}
Lapeer, D., Calzetti, D., Grasha, K., Adamo, A., Elmegreen, B.~G., Bik, A., Bortolini, G., Buckner, A., Cignoni, M., Correnti, M., Elmegreen, D.~M., Faustino~Vieira, H., Hamilton, M., Johnson, K., Lai, T. S.-Y., Linden, S.~T., Maji, S., Messa, M., Östlin, G., Pedrini, A., Sabbi, E., and Smith, L.~J. (2026).
\newblock Feast: Probing hierarchical star formation with the spatial distributions of young star clusters.
\newblock {\em The Astrophysical Journal}, 999(2):189.

\bibitem[Liddle and Loveday, 2008]{Lickgalaxycatalogue}
Liddle, A. and Loveday, J. (2008).
\newblock Lick galaxy catalogue.

\bibitem[Meng et~al., 2022]{Meng_Liu_Neiswanger_Song_Burke_Lobell_Ermon_2022}
Meng, C., Liu, E., Neiswanger, W., Song, J., Burke, M., Lobell, D., and Ermon, S. (2022).
\newblock Is-count: Large-scale object counting from satellite images with covariate-based importance sampling.
\newblock {\em Proceedings of the AAAI Conference on Artificial Intelligence}, 36(11):12034--12042.

\bibitem[{Messa} et~al., 2018]{2018MNRAS.473..996M}
{Messa}, M., {Adamo}, A., {{\"O}stlin}, G., {Calzetti}, D., {Grasha}, K., {Grebel}, E.~K., {Shabani}, F., {Chandar}, R., {Dale}, D.~A., {Dobbs}, C.~L., {Elmegreen}, B.~G., {Fumagalli}, M., {Gouliermis}, D.~A., {Kim}, H., {Smith}, L.~J., {Thilker}, D.~A., {Tosi}, M., {Ubeda}, L., {Walterbos}, R., {Whitmore}, B.~C., {Fedorenko}, K., {Mahadevan}, S., {Andrews}, J.~E., {Bright}, S.~N., {Cook}, D.~O., {Kahre}, L., {Nair}, P., {Pellerin}, A., {Ryon}, J.~E., {Ahmad}, S.~D., {Beale}, L.~P., {Brown}, K., {Clarkson}, D.~A., {Guidarelli}, G.~C., {Parziale}, R., {Turner}, J., and {Weber}, M. (2018).
\newblock {The young star cluster population of M51 with LEGUS - I. A comprehensive study of cluster formation and evolution}.
\newblock {\em Monthly Notices of the Royal Astronomical Society}, 473(1):996--1018.

\bibitem[Midzuno, 1952]{midzuno1952sampling}
Midzuno, H. (1952).
\newblock On the sampling system with probability proportional to sum of sizes.
\newblock {\em Annals of the Institute of Statistical Mathematics}, 3:99--107.

\bibitem[{Pedrini} et~al., 2025]{2025arXiv250901670P}
{Pedrini}, A., {Adamo}, A., {Bik}, A., {Calzetti}, D., {Linden}, S.~T., {Gregg}, B., {Bajaj}, V., {Ryon}, J.~E., {Buckner}, A. S.~M., {Bortolini}, G., {Cignoni}, M., {Correnti}, M., {Duarte-Cabral}, A., {Elmegreen}, B.~G., {Faustino Vieira}, H., {Gallagher}, J.~S., {Grasha}, K., {Johnson}, K.~E., {Krumholz}, M.~R., {Lapeer}, D., {Lai}, T. S.~Y., {Messa}, M., {{\"O}stlin}, G., {Roos}, L., {Smith}, L.~J., and {Tosi}, M. (2025).
\newblock {The near infrared SED of young star clusters in the FEAST galaxies: Missing ingredients at 1-5 $\mu$m}.
\newblock {\em arXiv e-prints}, page arXiv:2509.01670.

\bibitem[{Pedrini} et~al., 2024]{2024ApJ...971...32P}
{Pedrini}, A., {Adamo}, A., {Calzetti}, D., {Bik}, A., {Gregg}, B., {Linden}, S.~T., {Bajaj}, V., {Ryon}, J.~E., {Ali}, A.~A., {Bortolini}, G., {Correnti}, M., {Elmegreen}, B.~G., {Elmegreen}, D.~M., {Gallagher}, J.~S., {Grasha}, K., {Gutermuth}, R.~A., {Johnson}, K.~E., {Melinder}, J., {Messa}, M., {{\"O}stlin}, G., {Sabbi}, E., {Smith}, L.~J., {Tosi}, M., and {Faustino Vieira}, H. (2024).
\newblock {FEAST: Feedback in Emerging extragAlactic Star ClusTers: JWST Spots Polycyclic Aromatic Hydrocarbon Destruction in NGC 628 during the Emerging Phase of Star Formation}.
\newblock {\em The Astrophysical Journal}, 971(1):32.

\bibitem[Peebles, 1975]{peebles1975jagellonian}
Peebles, P. J.~E. (1975).
\newblock Statistical analysis of catalogs of extragalactic objects. vi -- the galaxy distribution in the jagellonian field.
\newblock {\em The Astrophysical Journal}, 196:647--651.

\bibitem[Perez et~al., 2024a]{Perez_Maji_Sheldon_2024}
Perez, G., Maji, S., and Sheldon, D. (2024a).
\newblock Discount: Counting in large image collections with detector-based importance sampling.
\newblock {\em Proceedings of the AAAI Conference on Artificial Intelligence}, 38(20):22294--22302.

\bibitem[P{\'{e}}rez et~al., 2021]{pmcmjas_apj2021}
P{\'{e}}rez, G., Messa, M., Calzetti, D., Maji, S., Jung, D.~E., Adamo, A., and Sirressi, M. (2021).
\newblock {StarcNet}: Machine learning for star cluster identification.
\newblock {\em The Astrophysical Journal}, 907(2):100.

\bibitem[Perez et~al., 2024b]{perez2024human}
Perez, G., Sheldon, D., Van~Horn, G., and Maji, S. (2024b).
\newblock Human-in-the-loop visual re-id for population size estimation.
\newblock In {\em European Conference on Computer Vision}, pages 185--202. Springer.

\bibitem[Raddick et~al., 2009]{raddick2009galaxy}
Raddick, J., Bracey, G., Gay, P., Lintott, C., Murray, P., Schawinski, K., Szalay, A., and Vandenberg, J. (2009).
\newblock Galaxy zoo: Exploring the motivations of citizen science volunteers.
\newblock {\em Astronomy Education Review}, 9.

\bibitem[Sekretenko and Gavrikov, 1998]{SEKRETENKO1998113}
Sekretenko, O. and Gavrikov, V. (1998).
\newblock Characterization of the tree spatial distribution in small plots using the pair correlation function.
\newblock {\em Forest Ecology and Management}, 102(2):113--120.

\bibitem[Seldner et~al., 1977]{seldner1977lick}
Seldner, M., Siebers, B., Groth, E.~J., and Peebles, P. J.~E. (1977).
\newblock New reduction of the lick catalog of galaxies.
\newblock {\em The Astronomical Journal}, 82:249--256, 313--314.

\bibitem[Sen, 1952]{sen1952present}
Sen, A.~R. (1952).
\newblock Present status of probability sampling and its use in the estimation of a characteristic.
\newblock {\em Econometrica}, 20:103.

\bibitem[Shane and Wirtanen, 1967]{shane1967lick}
Shane, C.~D. and Wirtanen, C.~A. (1967).
\newblock {\em Publications of the Lick Observatory}, volume~22.

\bibitem[Sun et~al., 2014]{sun2014reduced}
Sun, Y., Skidmore, A.~K., Wang, T., van Gils, H.~A., Wang, Q., Qing, B., and Ding, C. (2014).
\newblock Reduced dependence of crested ibis on winter-flooded rice fields: Implications for their conservation.
\newblock {\em PLoS One}, 9(5):e98690.

\bibitem[Whitmore et~al., 2021]{whitmore2021phangs}
Whitmore, B.~C., Lee, J.~C., Chandar, R., Thilker, D.~A., Hannon, S., Wei, W., Huerta, E.~A., et~al. (2021).
\newblock Star cluster classification in the phangs--hst survey: Comparison between human and machine learning approaches.
\newblock {\em Monthly Notices of the Royal Astronomical Society}, 506(4):5294--5317.

\bibitem[Wilhelm and Frey, 1996]{PhysRevLett.77.2581}
Wilhelm, J. and Frey, E. (1996).
\newblock Radial distribution function of semiflexible polymers.
\newblock {\em Phys. Rev. Lett.}, 77:2581--2584.

\bibitem[Zwicky et~al., 1968]{zwicky1968catalogue}
Zwicky, F., Herzog, E., and Wild, P. (1968).
\newblock {\em Catalogue of Galaxies and of Clusters of Galaxies}.
\newblock California Institute of Technology, Pasadena.
\newblock (CIT).

\end{thebibliography}

\onecolumn
\aistatstitle{Active Measurement of Two-Point Correlations: \\
Supplementary Materials}

\section{PROOFS}

\subsection{Proof of Proposition 2 for MC Variance}

We can express $\hat f(S)$ in terms of indicator variables for a uniformly sampled $k$-subset:
\begin{equation}
\hat{f}(S) 
= \frac{n(n-1)}{k(k-1)}
  \sum_{\{u,v\}\subseteq S} 
  f(u,v)\,\mathbf 1_{\{u,v\}\in S_k}.
\end{equation}

Thus,
\begin{equation}
\operatorname{Var}[\hat{f}(S)]
= 
\!\left(\frac{n(n-1)}{k(k-1)}\right)^{2}
\sum_{\{u,v\},\{w,z\}\subseteq S}
f(u,v)f(w,z)\,
\operatorname{Cov}\bigl(\mathbf 1_{\{u,v\}\in S_k},\mathbf 1_{\{w,z\}\in S_k}\bigr).
\end{equation}

By definition of covariance,
\begin{equation}
\operatorname{Cov}(\mathbf 1_{\{u,v\}\in S_k},\mathbf 1_{\{w,z\}\in S_k})
=\Pr(\{u,v\},\{w,z\}\in S_k)
-\Pr(\{u,v\}\in S_k)\Pr(\{w,z\}\in S_k).
\end{equation}

The second term is $\bigl(\binom{n-2}{k-2}/\binom{n}{k}\bigr)^{2}$.  
The first term depends on how many distinct elements appear among $u,v,w,z$ (2, 3, or 4).  
By the same counting logic we obtain $\binom{n-2}{k-2}/\binom{n}{k}$, $\binom{n-3}{k-3}/\binom{n}{k}$, or $\binom{n-4}{k-4}/\binom{n}{k}$ for the three cases.

Plugging these in and splitting the sum accordingly:
\begin{equation}
\begin{aligned}
\operatorname{Var}[\hat{f}(S)] 
= &\left( \frac{n(n-1)}{k(k-1)} \right)^{2} 
\Bigg[
\biggl(
\frac{\binom{n-2}{k-2}}{\binom{n}{k}}
-\!\!\left(\frac{\binom{n-2}{k-2}}{\binom{n}{k}}\right)^{2}\biggr)
\sum_{\{u,v\}} f(u,v)^{2}
\\[0.75em]
&\quad+
\biggl(
\frac{\binom{n-3}{k-3}}{\binom{n}{k}}
-\!\!\left(\frac{\binom{n-2}{k-2}}{\binom{n}{k}}\right)^{2}\biggr)
\sum_{|\{u,v,w\}|=3} f(u,v)f(v,w)
\\[0.75em]
&\quad+
\biggl(
\frac{\binom{n-4}{k-4}}{\binom{n}{k}}
-\!\!\left(\frac{\binom{n-2}{k-2}}{\binom{n}{k}}\right)^{2}\biggr)
\sum_{|\{u,v,w,z\}|=4} f(u,v)f(w,z)
\Bigg].
\end{aligned}
\end{equation}

Simplifying terms,
\begin{equation}
\begin{aligned}
\operatorname{Var}[\hat{f}(S)]
=&\left(\frac{n(n-1)}{k(k-1)}-1\right)
\sum_{\{u,v\}} f(u,v)^{2}
\\[0.75em]
&\quad+\left(\frac{n(n-1)}{k(k-1)}\frac{k-2}{n-2}-1\right)
\sum_{|\{u,v,w\}|=3} f(u,v)f(v,w)
\\[0.75em]
&\quad+\left(\frac{n(n-1)}{k(k-1)}
\frac{(k-2)(k-3)}{(n-2)(n-3)}-1\right)
\sum_{|\{u,v,w,z\}|=4} f(u,v)f(w,z).
\end{aligned}
\end{equation}

Counting the terms in each sum:
\begin{enumerate}
    \item The first sum has $\frac{n(n-1)}{2}$ pairs.
    \item For each $u,v$ in the second sum there are $2*(n-2)$ pairs for $w,z$, giving a total of $n(n-1)(n-2)$ ordered triples.
    \item For each $u, v$ in the third sum there are $\frac{(n-2)(n-3)}{2}$ pairs for $w,z$ for a total of $\frac{n(n-1)(n-2)(n-3)}{4}$ quadruples. 
\end{enumerate}
Rewriting in terms of expectations:
\begin{equation}
\begin{aligned}
\operatorname{Var}[\hat{f}(S)] 
=&\left(\frac{n(n-1)}{k(k-1)}-1\right)
\frac{n(n-1)}{2}\,
\mathbb{E}_{\{u,v\}}\bigl[f(u,v)^{2}\bigr]
\\[0.75em]
&\quad+\left(\frac{n(n-1)}{k(k-1)}\frac{k-2}{n-2}-1\right)
n(n-1)(n-2)\,
\mathbb{E}_{|\{u,v,w\}|=3}\bigl[f(u,v)f(v,w)\bigr]
\\[0.75em]
&\quad+\left(\frac{n(n-1)}{k(k-1)}
\frac{(k-2)(k-3)}{(n-2)(n-3)}-1\right)
\frac{n(n-1)(n-2)(n-3)}{4}\,
\mathbb{E}_{|\{u,v,w,z\}|=4}\bigl[f(u,v)f(w,z)\bigr].
\end{aligned}
\end{equation}

Finally, with shorthand
\begin{equation}
\begin{aligned}
E_1&=\mathbb{E}_{\{u,v\}}\bigl[f(u,v)^{2}\bigr],\\
E_2&=\mathbb{E}_{|\{u,v,w\}|=3}\bigl[f(u,v)f(v,w)\bigr],\\
E_3&=\mathbb{E}_{|\{u,v,w,z\}|=4}\bigl[f(u,v)f(w,z)\bigr],
\end{aligned}
\end{equation}
we can simplify and group terms
\begin{equation}
\boxed{\;
\operatorname{Var}[\hat{f}(S)] 
=\frac{n(n-1)(n-k)}{2k(k-1)}
\Bigl(2(E_{2}-E_{3})kn+(E_{1}-4E_{2}+3E_{3})(n+k-1)\Bigr)
\;}
\end{equation}

\subsection{Full Derivation of IS Variance}

We begin from the definition of the importance sampling estimator:
\begin{equation}
    \hat{f}_{\text{IS}}(S) = \frac{G\, f(S_k)}{g(S_k)}, 
    \quad S_k \sim q(S_k) = \frac{g(S_k)}{\binom{n-2}{k-2} G}.
\end{equation}

By definition of variance under the proposal distribution \( q \),
\begin{equation}
    \operatorname{Var}_q[\hat{f}_{\text{IS}}(S)] 
    = G^2\left(\mathbb{E}_q\left[\frac{f^2(S_k)}{g^2(S_k)}\right] 
    - \mathbb{E}_q\left[\frac{f(S_k)}{g(S_k)}\right]^2\right).
\end{equation}

We can rewrite this in terms of uniform expectations to obtain:
\begin{equation}
    \operatorname{Var}_q[\hat{f}_{\text{IS}}(S)]
    = G\frac{n(n-1)}{k(k-1)}\mathbb{E}\left[\frac{f^2(S_k)}{g(S_k)}\right]
    - \left(\frac{n(n-1)}{k(k-1)}\right)^2\mathbb{E}\left[f(S_k)\right]^2,
\end{equation}
where expectations are now taken uniformly over all subsets \( S_k \).

\subsubsection*{A. Second-Order Taylor Expansion of \( h(X,Y) = \frac{X}{Y} \)}

We approximate \( \mathbb{E}[X/Y] \) by expanding \( h(X,Y) = X/Y \) around \( (\mu_X, \mu_Y) = (\mathbb{E}[X], \mathbb{E}[Y]) \) to second order:

\begin{align}
h(X,Y)
&\approx h(\mu_X, \mu_Y)
+ (X-\mu_X)\,h_X(\mu_X,\mu_Y)
+ (Y-\mu_Y)\,h_Y(\mu_X,\mu_Y) \notag \\
&\quad + \frac{1}{2}\left[
(X-\mu_X)^2 h_{XX}(\mu_X,\mu_Y)
+ 2(X-\mu_X)(Y-\mu_Y) h_{XY}(\mu_X,\mu_Y)
+ (Y-\mu_Y)^2 h_{YY}(\mu_X,\mu_Y)
\right].
\end{align}

We compute the partial derivatives:
\begin{align}
h_X &= \frac{1}{Y}, 
&h_Y &= -\frac{X}{Y^2}, \\
h_{XX} &= 0, 
&h_{XY} &= -\frac{1}{Y^2},
&h_{YY} &= \frac{2X}{Y^3}.
\end{align}

Taking expectations, all first-order terms vanish because 
\( \mathbb{E}[X-\mu_X] = \mathbb{E}[Y-\mu_Y] = 0 \). Hence:
\begin{align}
\mathbb{E}[h(X,Y)]
&\approx \frac{\mu_X}{\mu_Y}
+ \frac{1}{2}\Bigl[
0 
+ 2(-\frac{1}{\mu_Y^2})\operatorname{Cov}(X,Y)
+ \frac{2\mu_X}{\mu_Y^3}\operatorname{Var}(Y)
\Bigr] \\
&= \frac{\mu_X}{\mu_Y}
- \frac{\operatorname{Cov}(X,Y)}{\mu_Y^2}
+ \frac{\mu_X\,\operatorname{Var}(Y)}{\mu_Y^3}.
\end{align}

We now apply this expansion to the case \( X = f^2(S_k) \), \( Y = g(S_k) \).

\subsubsection*{B. Taylor Approximation of \( \mathbb{E}\!\left[\frac{f^2(S_k)}{g(S_k)}\right] \)}

By substituting \( X=f^2(S_k) \) and \( Y=g(S_k) \) in the above formula, we obtain:
\begin{equation}
\mathbb{E}\!\left[\frac{f^2(S_k)}{g(S_k)}\right]
\approx 
\frac{\mathbb{E}[f^2(S_k)]}{\mathbb{E}[g(S_k)]}
-\frac{\operatorname{Cov}(f^2(S_k), g(S_k))}{\mathbb{E}[g(S_k)]^2}
+\frac{\mathbb{E}[f^2(S_k)]\,\operatorname{Var}[g(S_k)]}{\mathbb{E}[g(S_k)]^3}.
\end{equation}

Using the definition of covariance,
\begin{align}
\operatorname{Cov}(f^2(S_k), g(S_k))
&= \mathbb{E}[f^2(S_k)g(S_k)]
- \mathbb{E}[f^2(S_k)]\,\mathbb{E}[g(S_k)],
\end{align}
which gives
\begin{equation}
\mathbb{E}\!\left[\frac{f^2(S_k)}{g(S_k)}\right]
\approx 
2\frac{\mathbb{E}[f^2(S_k)]}{\mathbb{E}[g(S_k)]}
-\frac{\mathbb{E}[f^2(S_k)g(S_k)]}{\mathbb{E}[g(S_k)]^2}
+\frac{\mathbb{E}[f^2(S_k)]\,\operatorname{Var}[g(S_k)]}{\mathbb{E}[g(S_k)]^3}.
\end{equation}
We have full access to the predicted $g(S_k)$, so the mean and variance terms involving $g$ can be computed exactly. Thus we only need unbiased estimates of $f^2(S_k)$ and $f^2(S_k)g(S_k)$ to get an unbiased estimate of the entire Taylor expansion.

\subsubsection*{C. Reduced Forms of \( \mathbb{E}[f^2(S_k)] \), \(\mathbb{E}\left[f(S_k)\right]^2\),  and \( \mathbb{E}[f^2(S_k)g(S_k)] \)}

Using indicators, the subset-level functions can be written explicitly as
\[
f(S_k) = \sum_{u<v} f(u,v)\,\mathbf{1}_{u,v}(S_k),
\qquad
g(S_k) = \sum_{u<v} g(u,v)\,\mathbf{1}_{u,v}(S_k).
\]

\paragraph{Expansion of \(f^2(S_k)\).}
Expanding the square gives
\[
f^2(S_k)
= \sum_{u<v} \sum_{w<z}
    f(u,v)f(w,z)\,
    \mathbf{1}_{u,v}(S_k)\mathbf{1}_{w,z}(S_k).
\]
Taking the uniform expectation over all \(k\)-subsets \(S_k\),
\begin{equation}\label{eq:f2-uv-expect}
\mathbb{E}[f^2(S_k)]
= \sum_{\substack{u<v,\,w<z}} 
    f(u,v)f(w,z)\,\Pr(\{u,v,w,z\}\subset S_k).
\end{equation}

The inclusion probabilities depend only on the number of distinct vertices in the union
\(\{u,v,w,z\}\):
\[
t = |\{u,v,w,z\}| \in \{2,3,4\}.
\]
Specifically:
\begin{align*}
C_1 &= \Pr(\{u,v\}\subset S_k) 
    = \frac{\binom{n-2}{k-2}}{\binom{n}{k}}, &&\text{(2 distinct vertices)}\\[4pt]
C_2 &= \Pr(\{u,v,w\}\subset S_k)
    = \frac{\binom{n-3}{k-3}}{\binom{n}{k}}, &&\text{(3 distinct vertices, one shared)}\\[4pt]
C_3 &= \Pr(\{u,v,w,z\}\subset S_k)
    = \frac{\binom{n-4}{k-4}}{\binom{n}{k}}, &&\text{(4 distinct vertices, disjoint pairs)}.
\end{align*}

Grouping the terms of \eqref{eq:f2-uv-expect} by these cases gives
\begin{equation}\label{eq:f2-three-term-uv}
\mathbb{E}[f^2(S_k)]
= C_1 \sum_{u<v} f(u,v)^2
+ C_2 \sum_{\substack{u<v,\,w<z:\\|\{u,v,w\}|=3}} f(u,v)f(v,w)
+ C_3 \sum_{\substack{u<v,\,w<z:\\|\{u,v,w,z\}|=4}} f(u,v)f(w,z).
\end{equation}

These ratios simplify algebraically to
\[
C_1=\frac{k(k-1)}{n(n-1)},\qquad
C_2=\frac{k(k-1)(k-2)}{n(n-1)(n-2)},\qquad
C_3=\frac{k(k-1)(k-2)(k-3)}{n(n-1)(n-2)(n-3)}.
\]

We can also count the number of terms in each sum to see that,
\[
\sum_{u<v} f(u,v)^2 \;=\; \frac{n(n-1)}{2}\,E_1,
\]
\[
\sum_{\substack{u<v,\,w<z\\ |\{u,v,w\}|=3}} f(u,v)f(v,w)
\;=\; n(n-1)(n-2)\,E_2,
\]
\[
\sum_{\substack{u<v,\,w<z\\ |\{u,v,w,z\}|=4}} f(u,v)f(w,z)
\;=\; \frac{n(n-1)(n-2)(n-3)}{4}\,E_3.
\]
Where $E_1$, $E_2$, and $E_3$ are defined as the expectations:

\begin{equation}  
\begin{aligned}  
E_2 &= \mathbb{E}_{\{u,v\}}\bigl[f(u,v)^{2}\bigr], \\  
E_3 &= \mathbb{E}_{|\{u,v,w\}|=3}\bigl[f(u,v)f(v,w)\bigr], \\  
E_4 &= \mathbb{E}_{|\{u,v,w,z\}|=4}\bigl[f(u,v)f(w,z)\bigr].  
\end{aligned}  
\end{equation}  

After substitution and cancellation the result is:

\[
\boxed{%
\mathbb{E}[f^2(S_k)]
= \frac{k(k-1)}{2}\,E_1
\;+\; k(k-1)(k-2)\,E_2
\;+\; \frac{k(k-1)(k-2)(k-3)}{4}\,E_3.
}
\]

\subsection*{Expansion of \(\mathbb{E}[f(S_k)]^2\)}

Using the indicator representation, the uniform expectation of \(f(S_k)\) is
\[
\mathbb{E}[f(S_k)]
= \sum_{u<v} f(u,v)\,\Pr(\{u,v\}\subset S_k)
= C_1\sum_{u<v} f(u,v).
\]

We now square this quantity to obtain \(\mathbb{E}[f(S_k)]^2\):
\begin{align}
\mathbb{E}[f(S_k)]^2
&= C_1^2 \left(\sum_{u<v} f(u,v)\right)^2 \notag\\[4pt]
&= C_1^2 \left(\sum_{\substack{u<v,\,w<z}} f(u,v)f(w,z) \right).
\label{eq:Ef2-expanded}
\end{align}

Split the double sum according to the number of distinct vertices in \(\{u,v,w,z\}\). Thus we can write \(\mathbb{E}[f(S_k)]^2\) in the fully split form
\begin{equation}\label{eq:Ef2-Cs}
\mathbb{E}[f(S_k)]^2
= C_1^2 \Bigg(
\sum_{u<v} f(u,v)^2
\;+\; \sum_{\substack{u<v,\,w<z\\ |\{u,v,w\}|=3}} f(u,v)f(v,w)
\;+\; \sum_{\substack{u<v,\,w<z\\ |\{u,v,w,z\}|=4}} f(u,v)f(w,z)
\Bigg).
\end{equation}
Following similar substitutions as before,

\begin{equation}
\boxed{
\mathbb{E}[f(S_k)]^2
= \frac{k^2(k-1)^2}{n(n-1)}\left(\frac{1}{2}E_1 + (n-2)E_2 + \frac{(n-2)(n-3)}{4}E_3\right)}
\end{equation}

\paragraph{Expansion of \(f^2(S_k)g(S_k)\).}
Similarly, expand the product \(f^2(S_k)g(S_k)\):
\[
f^2(S_k)g(S_k)
= \sum_{u<v}\sum_{w<z}\sum_{a<b}
    f(u,v)f(w,z)g(a,b)\,
    \mathbf{1}_{u,v}(S_k)\mathbf{1}_{w,z}(S_k)\mathbf{1}_{a,b}(S_k).
\]
Taking expectation gives
\begin{equation}\label{eq:f2g-uvwzab}
\mathbb{E}[f^2(S_k)g(S_k)]
= \sum_{u<v}\sum_{w<z}\sum_{a<b}
    f(u,v)f(w,z)g(a,b)\;
    \Pr\big(\{u,v,w,z,a,b\}\subset S_k\big).
\end{equation}

Let
\[
t = |\{u,v,w,z,a,b\}|
\]
denote the number of distinct vertices covered by the three pairs \((u,v), (w,z), (a,b)\).
Clearly \(t\in\{2,3,4,5,6\}\), and the inclusion probability depends only on \(t\):
\[
\Pr\big(\{u,v,w,z,a,b\}\subset S_k\big)
= \frac{\binom{n-t}{k-t}}{\binom{n}{k}}.
\]

Grouping \eqref{eq:f2g-uvwzab} by \(t\) gives the exact decomposition
\begin{equation}\label{eq:f2g-t}
\mathbb{E}[f^2(S_k)g(S_k)]
= \sum_{t=2}^{6} \frac{\binom{n-t}{k-t}}{\binom{n}{k}}
   \sum_{\substack{u<v,\,w<z,\,a<b\\|\{u,v,w,z,a,b\}|=t}}
        f(u,v)f(w,z)g(a,b).
\end{equation}

Each term corresponds to a specific overlap pattern among the three pairs:
\begin{itemize}
  \item \(t=2:\) all three pairs are identical \((u,v)=(w,z)=(a,b)\);
  \item \(t=3:\) three pairs covering exactly three vertices (e.g.\ a star or triangle);
  \item \(t=4:\) configurations where the three pairs span four distinct vertices;
  \item \(t=5:\) configurations sharing exactly one vertex across otherwise disjoint pairs;
  \item \(t=6:\) all three pairs disjoint.
\end{itemize}

Define \(D_t\) as the average of \(f(u,v)f(w,z)g(a,b)\) over all ordered triplets \((u<v,w<z,a<b)\) whose union has size \(t\), and let \(M_t\) be the corresponding number of ordered triplets. Then
\[
\sum_{\substack{u<v,\,w<z,\,a<b\\|\{u,v,w,z,a,b\}|=t}} f(u,v)f(w,z)g(a,b) = M_t\,D_t,
\]
and
\[
\mathbb{E}[f^2(S_k)g(S_k)]
= \sum_{t=2}^{6} \frac{\binom{n-t}{k-t}}{\binom{n}{k}}\,M_t\,D_t.
\]

The integer multiplicities \(M_t\) (number of ordered triplets of pairs whose union covers exactly \(t\) distinct vertices) are:
\[
\begin{aligned}
M_2 &= \binom{n}{2} \;=\; \frac{n(n-1)}{2},\\[4pt]
M_3 &= 24\binom{n}{3} \;=\; 4\,n(n-1)(n-2),\\[4pt]
M_4 &= 114\binom{n}{4} \;=\; \frac{19}{4}\,n(n-1)(n-2)(n-3),\\[4pt]
M_5 &= 180\binom{n}{5} \;=\; \tfrac{3}{2}\,n(n-1)(n-2)(n-3)(n-4),\\[4pt]
M_6 &= 90\binom{n}{6} \;=\; \tfrac{1}{8}\,n(n-1)(n-2)(n-3)(n-4)(n-5).
\end{aligned}
\]
Writing the full expression out and simplifying,
\[
\boxed{%
\mathbb{E}\big[f^2(S_k)g(S_k)\big]
= k(k-1)\Bigg[
\frac{1}{2}D_2
+(k-2)\Bigg(
4D_3
+(k-3)\Bigg(
\frac{19}{4}D_4
+(k-4)\Big(\frac{3}{2}D_5+\frac{k-5}{8}D_6\Big)
\Bigg)\Bigg)\Bigg].
}
\]

\subsection*{Derivation of triplet multiplicities \(M_t\) (ordered triplets of pairs)}

Fix a vertex set \(V\) of size \(t\), \(2\le t\le 6\).  
We ask: how many ordered triplets of unordered pairs
\((u<v,\;w<z,\;a<b)\) chosen from the pairs on \(V\) have the property that
the union of the three pairs equals \(V\) (i.e. covers exactly the \(t\) vertices)?

Let \(P(t)=\binom{t}{2}\) denote the number of unordered pairs inside a
\(t\)-vertex set. The total number of ordered triplets of pairs drawn from \(V\)
(with repetition allowed) is \(P(t)^3\). We must subtract those triplets that
lie entirely inside some proper subset $U \subset V$.  Inclusion--exclusion
over which vertices of \(V\) are omitted gives a compact, general formula.

For \(i=0,1,\dots,t-2\) choose \(i\) vertices of \(V\) to omit; then the
intersection of those choices is a set of size \(t-i\).  The number of ordered
triplets of pairs contained in that \((t-i)\)-vertex intersection equals
\(\binom{t-i}{2}^3\). There are \(\binom{t}{i}\) ways to choose which \(i\)
vertices to omit. By inclusion--exclusion the number of ordered triplets of
pairs whose union is exactly \(V\) is
\[
\sum_{i=0}^{\,t-2} (-1)^i \binom{t}{i}\,\binom{t-i}{2}^3.
\]
(We stop at \(i=t-2\) because when \(t-i<2\) there are no pairs.)

Thus the per-\(t\)-set ordered-triplet count equals the above sum, and the
total number across the \(n\)-vertex ground set is that per-\(t\)-set count
times the number of choices of the \(t\)-set, \(\binom{n}{t}\). Hence
\[
M_t \;=\; \binom{n}{t}\;\sum_{i=0}^{\,t-2} (-1)^i \binom{t}{i}\,\binom{t-i}{2}^3,
\qquad t=2,\dots,6.
\]
Evaluating for each $t$ gives the values above.

\subsection{Unbiased Estimators for Pair-wise Expectations}

\paragraph{A. Proof of Proposition 3: Estimating $E_i$ for MC}
In the main text we claim that for the simple MC the sample mean is an unbiased estimate of each $E_i$, giving the following estimator:

\begin{equation}
\boxed{
    \hat{E_i} = \frac{1}{n_i}\sum_{\substack{\{u,v,w,z\} \subseteq S_k \\
    u \neq v, w \neq z\\
    |\{u,v,w,z\}| = i}} f(u,v)f(w,z), 
\quad S_k \sim U.}
\end{equation}
where $n_i$ is the number of terms in the sum. To prove this, we rewrite in terms of indicators over uniformly sampled $k$-subsets,

\begin{equation}
    \mathbb{E}[\hat{E_i}] = \mathbb{E}\bigg[\frac{1}{n_i}\sum_{\substack{\{u,v,w,z\} \subseteq S \\
    u \neq v, w \neq z\\
    |\{u,v,w,z\}| = i}} f(u,v)f(w,z)\mathbf{1}_{u,v}(S_k)\mathbf{1}_{w,z}(S_k)\bigg], 
\quad S_k \sim U.
\end{equation}

Case $i=2$

This corresponds to the same pair repeated: $(u,v)=(w,z)$.

\[
\mathbb{E}[\hat{E_2}] 
= \frac{1}{n_2}\sum_{u<v} f(u,v)^2 \, \mathbb{E}[\mathbf{1}_{u,v}(S_k)].
\]
Under uniform sampling of $k$ elements from $n$,
\[
\mathbb{E}[\mathbf{1}_{u,v}(S_k)] 
= \Pr(u,v \in S_k)
= \frac{\binom{n-2}{k-2}}{\binom{n}{k}} 
= \frac{k(k-1)}{n(n-1)}.
\]
Also,
\[
n_2 = \binom{k}{2}.
\]
Substituting:
\[
\mathbb{E}[\hat{E_2}] 
= \frac{1}{\binom{k}{2}} \sum_{u<v} f(u,v)^2 \frac{k(k-1)}{n(n-1)}
= \frac{1}{\binom{n}{2}}\sum_{u<v} f(u,v)^2
= E_2.
\]
Thus $\hat{E_2}$ is unbiased.

Case $i=3$

This corresponds to pairs sharing exactly one vertex, e.g.\ $(u,v)$ and $(u,w)$ with $u,v,w$ distinct.

The estimator is
\[
\hat{E_3}
= \frac{1}{n_3}\sum_{\substack{u,v,w\,\text{distinct}}} f(u,v)f(u,w)\,
\mathbf{1}_{u,v}(S_k)\mathbf{1}_{u,w}(S_k).
\]
Take expectation:
\[
\mathbb{E}[\hat{E_3}]
= \frac{1}{n_3}\sum_{u,v,w} f(u,v)f(u,w)\,
\mathbb{E}[\mathbf{1}_{u,v}(S_k)\mathbf{1}_{u,w}(S_k)].
\]
The two pairs share three distinct vertices $\{u,v,w\}$, so
\[
\Pr(u,v,w \in S_k)
= \frac{\binom{n-3}{k-3}}{\binom{n}{k}}
= \frac{k(k-1)(k-2)}{n(n-1)(n-2)}.
\]
The number of such terms is
\[
n_3 = k(k-1)(k-2).
\]
Substitute:
\[
\mathbb{E}[\hat{E_3}]
= \frac{1}{k(k-1)(k-2)} 
\sum_{u,v,w} f(u,v)f(u,w)
\frac{k(k-1)(k-2)}{n(n-1)(n-2)}
= \frac{1}{n(n-1)(n-2)}\sum_{u,v,w} f(u,v)f(u,w)
= E_3.
\]
Hence $\hat{E_3}$ is unbiased.

Case $i=4$

This corresponds to two disjoint pairs $(u,v)$ and $(w,z)$ with all vertices distinct.

\[
\hat{E_4}
= \frac{1}{n_4}\sum_{\substack{u,v,w,z\,\text{distinct}}}
f(u,v)f(w,z)\,\mathbf{1}_{u,v}(S_k)\mathbf{1}_{w,z}(S_k).
\]
Take expectation:
\[
\mathbb{E}[\hat{E_4}]
= \frac{1}{n_4}\sum_{u,v,w,z} f(u,v)f(w,z)\,
\Pr(u,v,w,z \in S_k).
\]
Since there are four distinct vertices,
\[
\Pr(u,v,w,z \in S_k)
= \frac{\binom{n-4}{k-4}}{\binom{n}{k}}
= \frac{k(k-1)(k-2)(k-3)}{n(n-1)(n-2)(n-3)}.
\]
The number of disjoint-pair terms is
\[
n_4 = \frac{k(k-1)(k-2)(k-3)}{4}.
\]
Substituting,
\[
\mathbb{E}[\hat{E_4}]
= \frac{4}{n(n-1)(n-2)(n-3)}
\sum_{u,v,w,z} f(u,v)f(w,z)
= E_4.
\]
Therefore $\hat{E_4}$ is unbiased.

In each case,
\[
\mathbb{E}[\hat{E_i}] = E_i, \quad i=2,3,4.
\]
Thus the sample mean over uniformly drawn $k$-subsets is an unbiased estimator of the corresponding population moment.

\paragraph{B. Estimating $E_i$ for IS}
For the IS estimator, the $k$-subsets are not drawn uniformly, so we cannot directly estimate the uniform expectations $E_i$ and $D_i$ using sample means. Instead, we employ the horvitz-thompson estimators,
\begin{equation}
\boxed{
    \hat{E}_i^{\text{IS}} = \frac{1}{N_i}\sum_{\substack{u,v,w,z \subseteq S_k \\
    u \neq v, w \neq z\\
    |\{u,v,w,z\}| = i}} \frac{f(u,v)f(w,z)}{\pi(u,v,w,z)}, 
\quad S_k \sim q,}
\end{equation}

where $\pi(u,v,w,z) = \Pr(u,v,w,z \in S_k)$, and $N_i$ is the number of such pairs across the entire dataset. The unbiasedness of this estimator is easy to show:

\begin{equation}
\begin{aligned}
\mathbb{E}[\hat{E}_i^{\text{IS}}] 
&= \frac{1}{N_i}
\sum_{\substack{u,v,w,z \subseteq S \\ u \neq v,\, w \neq z \\ |\{u,v,w,z\}| = i}}
\frac{f(u,v)f(w,z)}{\pi(u,v,w,z)}\,
\mathbb{E}[\mathbf{1}_{u,v}(S_k)\mathbf{1}_{w,z}(S_k)] \\[6pt]
&= \frac{1}{N_i}
\sum_{\substack{u,v,w,z \subseteq S \\ u \neq v,\, w \neq z \\ |\{u,v,w,z\}| = i}}
\frac{f(u,v)f(w,z)}{\pi(u,v,w,z)}\,\pi(u,v,w,z) \\[6pt]
&= \frac{1}{N_i}
\sum_{\substack{u,v,w,z \subseteq S \\ u \neq v,\, w \neq z \\ |\{u,v,w,z\}| = i}}
f(u,v)f(w,z) = E_i.
\end{aligned}
\end{equation}

To calculate $\pi(u,v,w,z)$ explicitly, we can sum over all $k$-subsets containing $u,v,w,z$. Recall that the proposal is defined as,
\[
q(S_k) \;=\; \frac{g(S_k)}{\binom{n-2}{k-2}\,G},
\qquad
G := \sum_{\substack{a<b\\ a,b\in S}} g(a,b),
\]

Let \(R=\{u,v,w,z\}\) and $t=|R|$ be the number of unique vertices. Taking the sum over all $k$-subsets containing $u,v,w,z$,
\[
\pi(u,v,w,z)
= \sum_{S_k:\;R\subset S_k} q(S_k)
= \frac{1}{\binom{n-2}{k-2}\,G}\sum_{S_k:\;R\subset S_k} g(S_k).
\]

Every \(k\)-subset \(S_k\) that contains \(R\) can be written as
\(S_k=R\cup S'\) with \(S'\subset S\setminus R\).
Write \(g(S_k)=\sum_{a<b,\,a,b\in S_k} g(a,b)\) and split the inner pair-sum
into three disjoint types: pairs entirely inside \(R\), cross pairs with one
endpoint in \(R\) and one in \(S'\), and pairs entirely inside \(S'\). Hence

\[
\begin{aligned}
\sum_{S_k:\;R\subset S_k} g(S_k)
&= \sum_{\substack{S'\subset S\setminus R\\|S'|=k-t}}
    \sum_{\substack{a<b\\ a,b\in R\cup S'}} g(a,b) \\[4pt]
&= \binom{n-t}{k-t}\sum_{\substack{a<b\\ a,b\in R}} g(a,b)
   \;+\; \binom{n-t-1}{k-t-1}\sum_{\substack{r\in R\\ x\notin R}} g(r,x) \\[4pt]
&\qquad\qquad\qquad\qquad
   +\; \binom{n-t-2}{k-t-2}\sum_{\substack{x<y\\ x,y\notin R}} g(x,y).
\end{aligned}
\]

Substituting into the expression for \(\pi(u,v,w,z)\) gives

\[
\boxed{
\pi(u,v,w,z)
= \frac{1}{\binom{n-2}{k-2}\,G}\Bigg[
\binom{n-t}{k-t}\,\sum_{\substack{a<b\\ a,b\in R}} g(a,b)
+ \binom{n-t-1}{k-t-1}\,\sum_{\substack{r\in R\\ x\notin R}} g(r,x)
+ \binom{n-t-2}{k-t-2}\,\sum_{\substack{x<y\\ x,y\notin R}} g(x,y)
\Bigg].}
\]
We can compute this efficiently for all pairs of pairs by pre-computing the overall sum $G$ and all per-vertex sums. Using inclusion-exclusion arguments we can then compute $\pi(u,v,w,z)$ in constant time.

\paragraph{C. Estimating $D_i$ for IS}
We can also use horvitz-thompson estimators for $D_i$. The naive approach would be,
\begin{equation}
    \frac{1}{M_i}\sum_{\substack{u,v,w,z,a,b \subseteq S_k \\
    u \neq v, w \neq z, a \neq b\\
    |\{u,v,w,z,a,b\}| = i}} \frac{f(u,v)f(w,z)g(a,b)}{\pi(u,v,w,z,a,b)}, 
\quad S_k \sim q,
\end{equation}
where $M_i$ is the number of triplets of pairs as defined previously. However, we know the value of $g(a,b)$ for all pairs since it is an estimated count. Thus the probability of seeing a particular triplet only depends on the first pair of pairs $u,v,w,z$ being in $S_k$, which is the same $\pi(u,v,w,z)$ as before. This gives our $D_i$ estimator,
\begin{equation}
    \boxed{
    \hat{D}_i^{\text{IS}} = \frac{1}{M_i}\sum_{\substack{u,v,w,z,a,b \subseteq S_k \\
    u \neq v, w \neq z, a \neq b\\
    |\{u,v,w,z,a,b\}| = i}} \frac{f(u,v)f(w,z)g(a,b)}{\pi(u,v,w,z)}, 
\quad S_k \sim q,}
\end{equation}

with $\pi(u,v,w,z)$ defined above and $M_i$ defined at the end of section 1.2.C.

Since only the $g(a,b)$ term depends on $a$ and $b$, we can move the sum over $a,b$ inside. Similar to our trick for $\pi(u,v,w,z)$, we can use inclusion-exclusion arguments with pre-computed sums to calculate the $\sum_{a,b} g(a,b)$ term in constant time. Thus computing $\hat{D}_i^{\text{IS}}$ has the same time complexity as $\hat{E}_i^{\text{IS}}$, which is quadratic in the number of true edges.

\paragraph{Note on Time Complexity}
Although the proposed Horvitz–Thompson estimators require quadratic time in the number of true edges, this computation does not need to be repeated at every iteration. In practice, the quantities $E_i$ and $D_i$ can be estimated once and reused across iterations with different values of $k$. Since the variance estimate is a linear function of $E_i$ and $D_i$, each subsequent iteration then requires only constant computation. Moreover, alternative unbiased estimators for $E_i$ and $D_i$ could further reduce the computational cost (for instance, by randomly sampling pairs of pairs rather than evaluating all of them) though this may introduce a tradeoff between speed and accuracy.

\section{TRAINING DETAILS AND COMPUTE}
We use the StarcNet model as our classifier. The number of input channels is reduced to 8 to match the JWST dataset, while the rest of the architecture remains unchanged. The model is trained using input patches of 24 pixels, a batch size of 64, a learning rate of 1e-4, and for a total of 20 epochs. We use the Adam optimizer and decrease the learning rate to 1e-5 after 10 epochs. These settings closely match those used in the StarcNet codebase.

We perform 10-fold cross-validation: the dataset is divided into 10 folds, and the model is trained 10 times, each time using 9 folds for training and 1 distinct fold for validation. The predictions from each held-out fold are then aggregated across all 10 trained models to produce the final classifier outputs. Training all models takes less than two hours on a single NVIDIA 2080 Ti GPU.

Our estimation experiments were performed on a CPU cluster. For the larger NGC 628 galaxy, running just the MC and IS estimators took about 10 minutes per trial on a single CPU core. Computing the variance estimates increased the total time to about 20 minutes and 90 minutes respectively. For the NGC 4449 galaxy, each trial took at most a few minutes.

\section{ADDITIONAL RESULTS}
\subsection{Coverage and Radius Plots}
Here we present coverage and radius plots for all 13 bins for NGC 4449 and 14 bins for NGC628. We see good coverage in most cases, with poor coverage occurring when the number of edges is too small to get a reliable variance estimate. We also see that the delta method approximation reduces the coverage for the IS estimator, although it still converges to the optimal 95\% given sufficient data. The count of edges in each bin spans multiple orders of magnitude, with the largest bins having between 100x and 1000x the edges of the smallest. Refer to Figure 1 in the main text for the ground truth edge count plotted on a log scale.
\begin{figure}[htbp]
    \centering
    \includegraphics[width=0.43\linewidth]{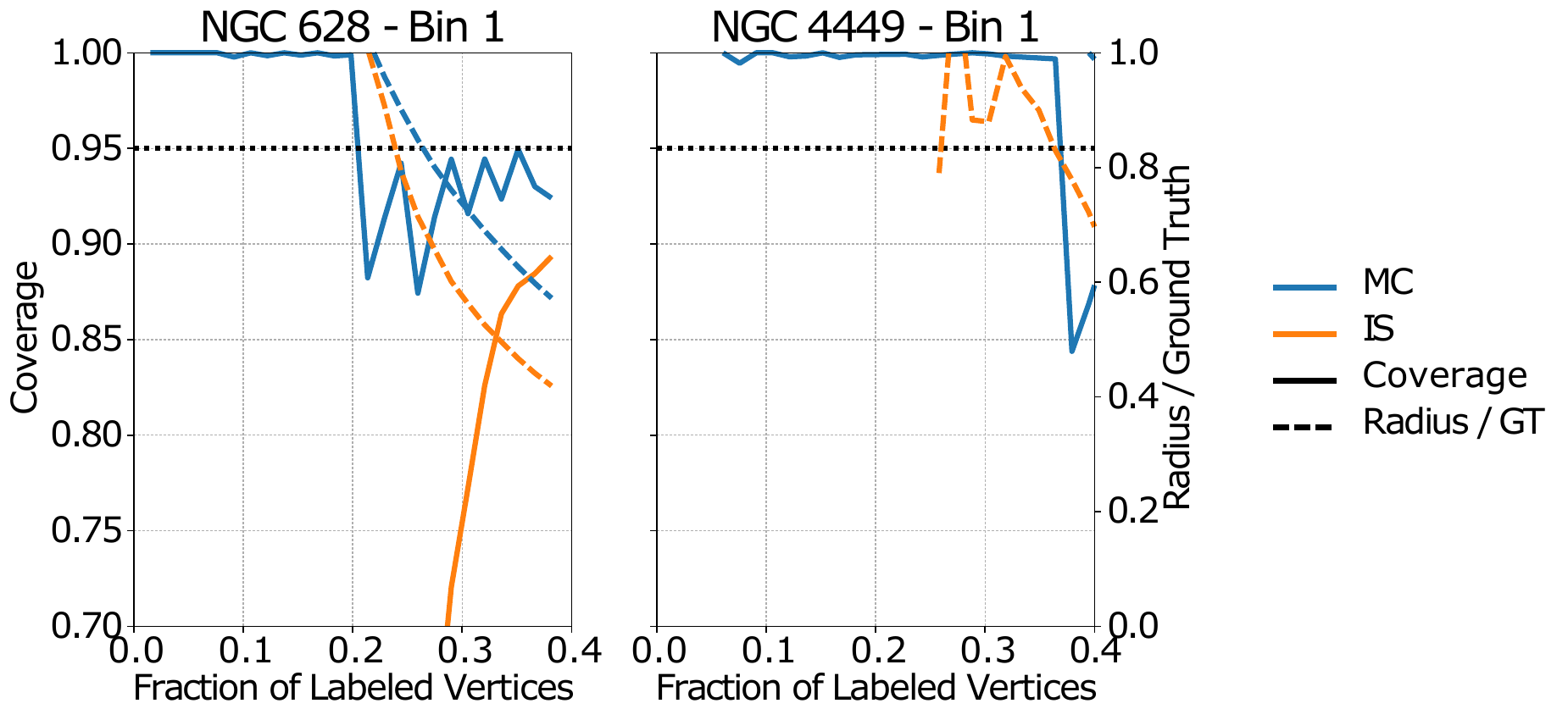}
    \includegraphics[width=0.43\linewidth]{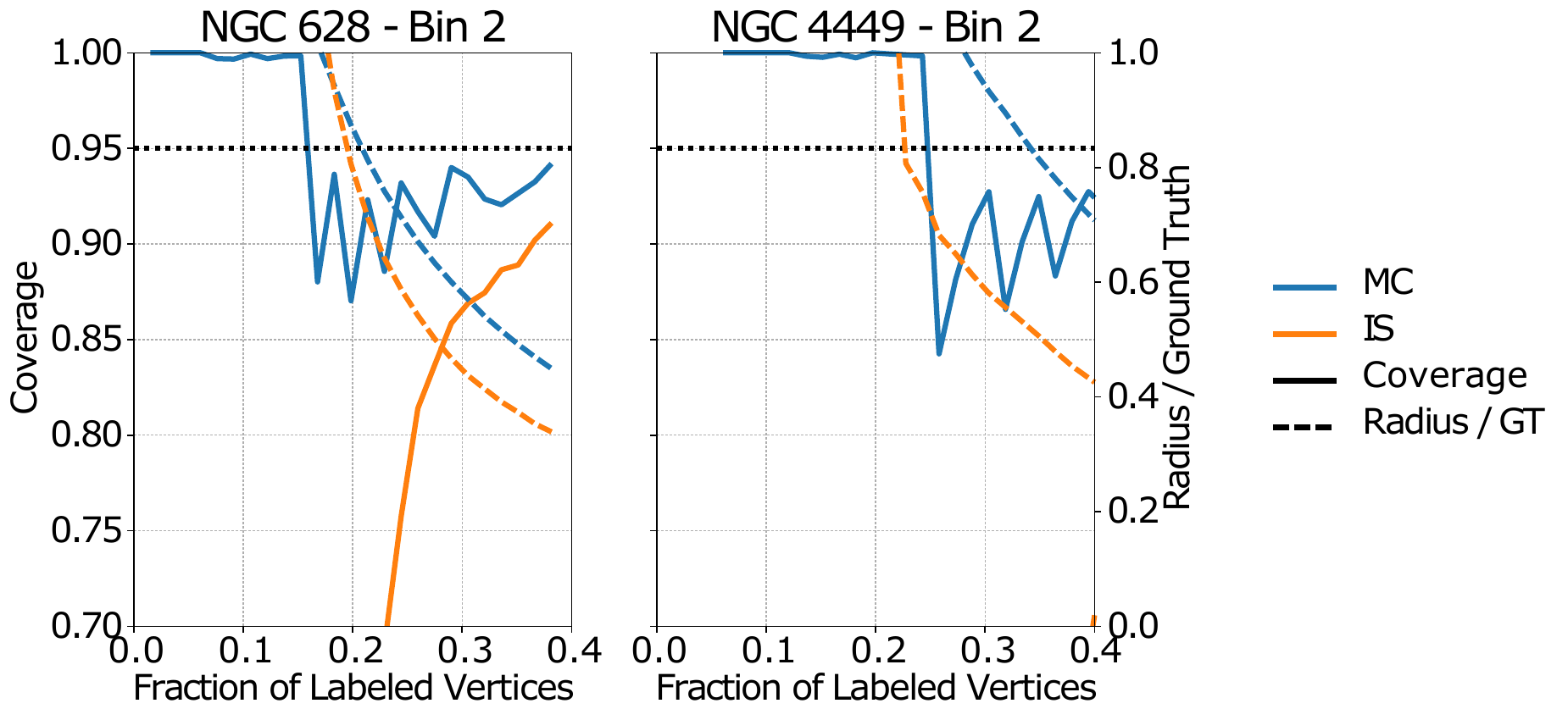}
    \includegraphics[width=0.43\linewidth]{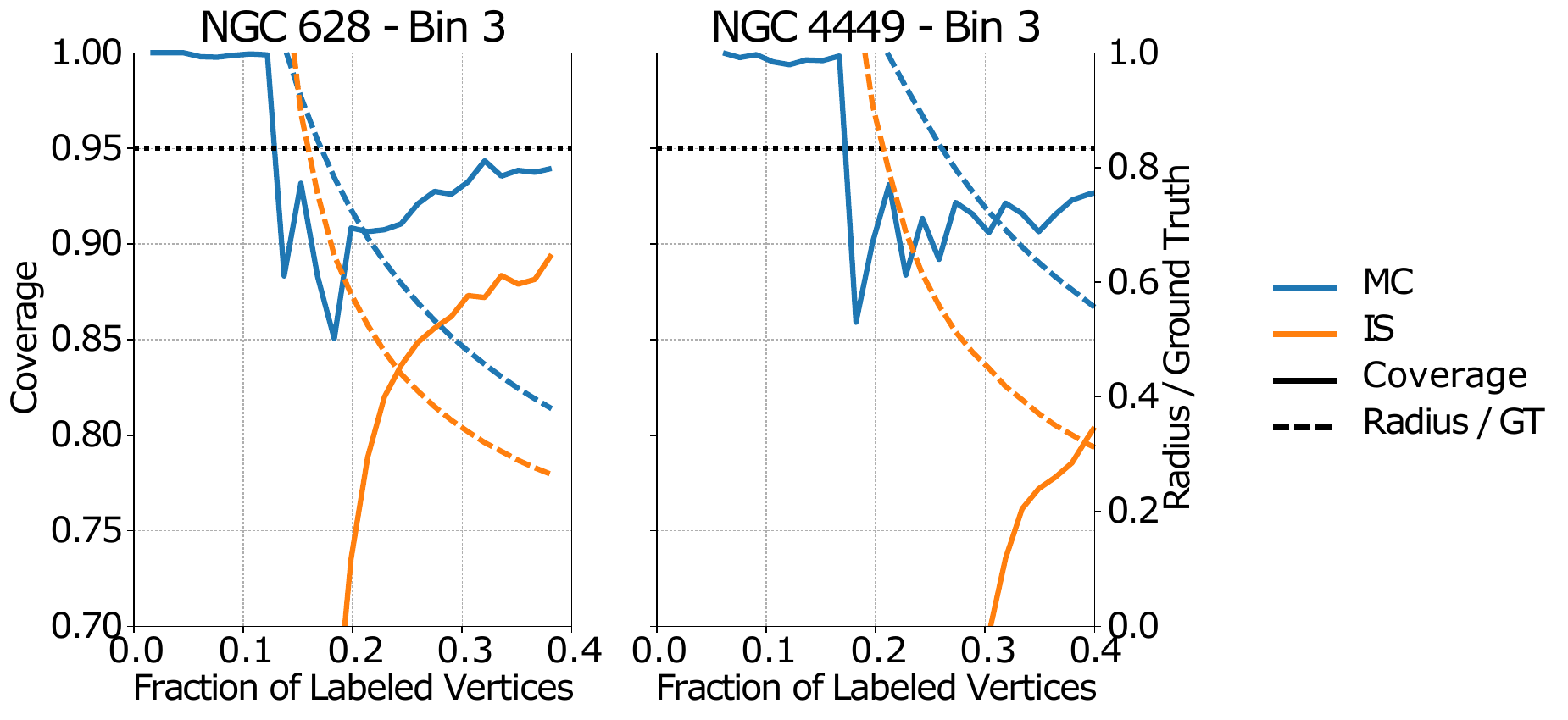}
    \includegraphics[width=0.43\linewidth]{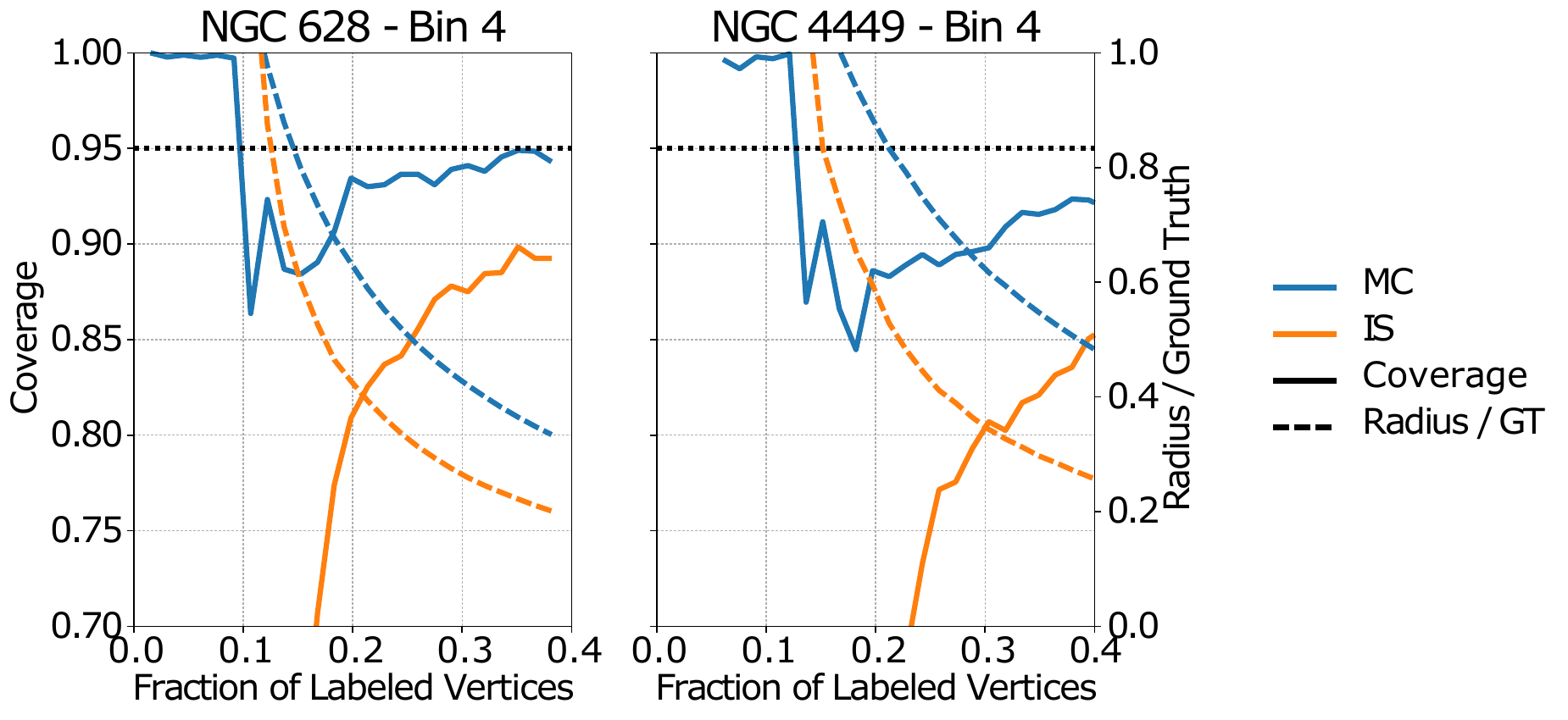}
    \includegraphics[width=0.43\linewidth]{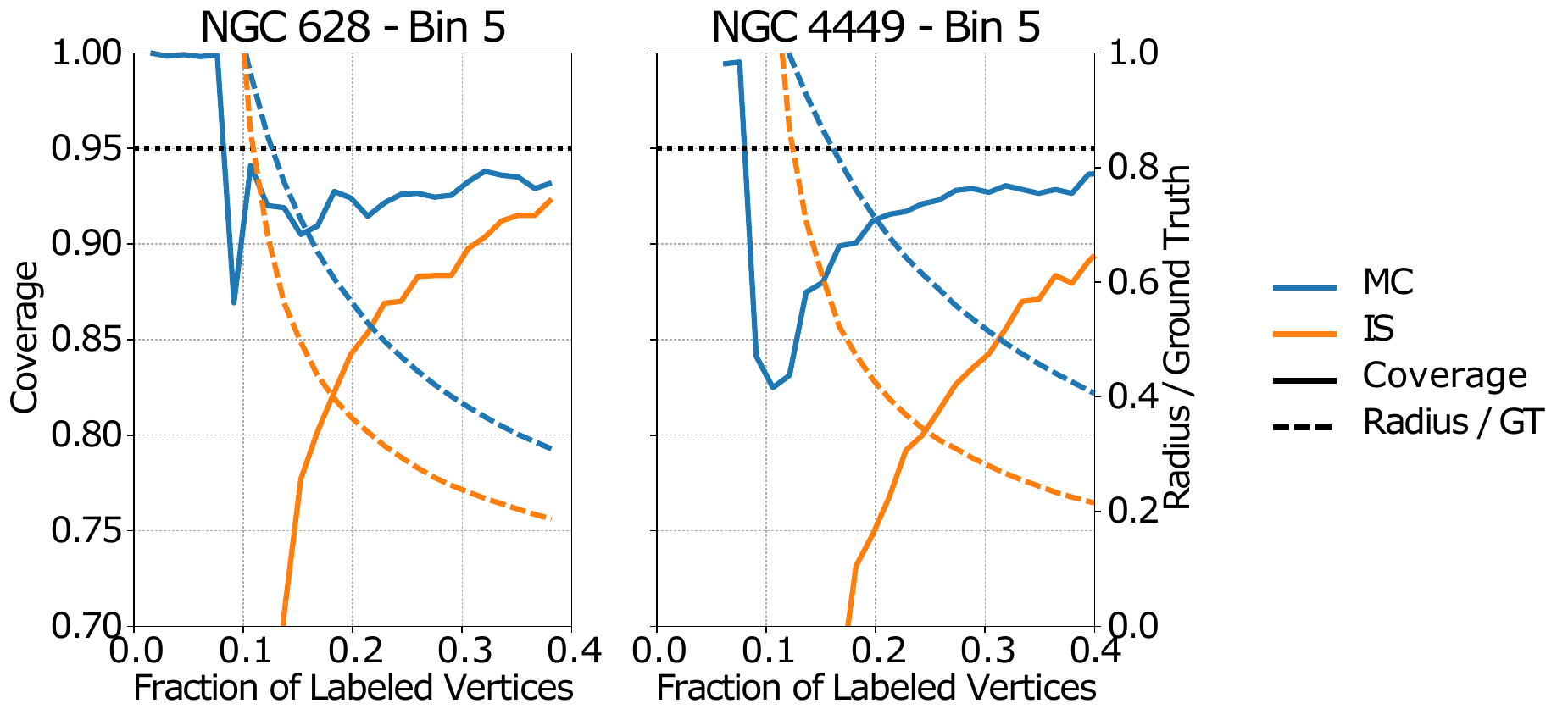}
    \includegraphics[width=0.43\linewidth]{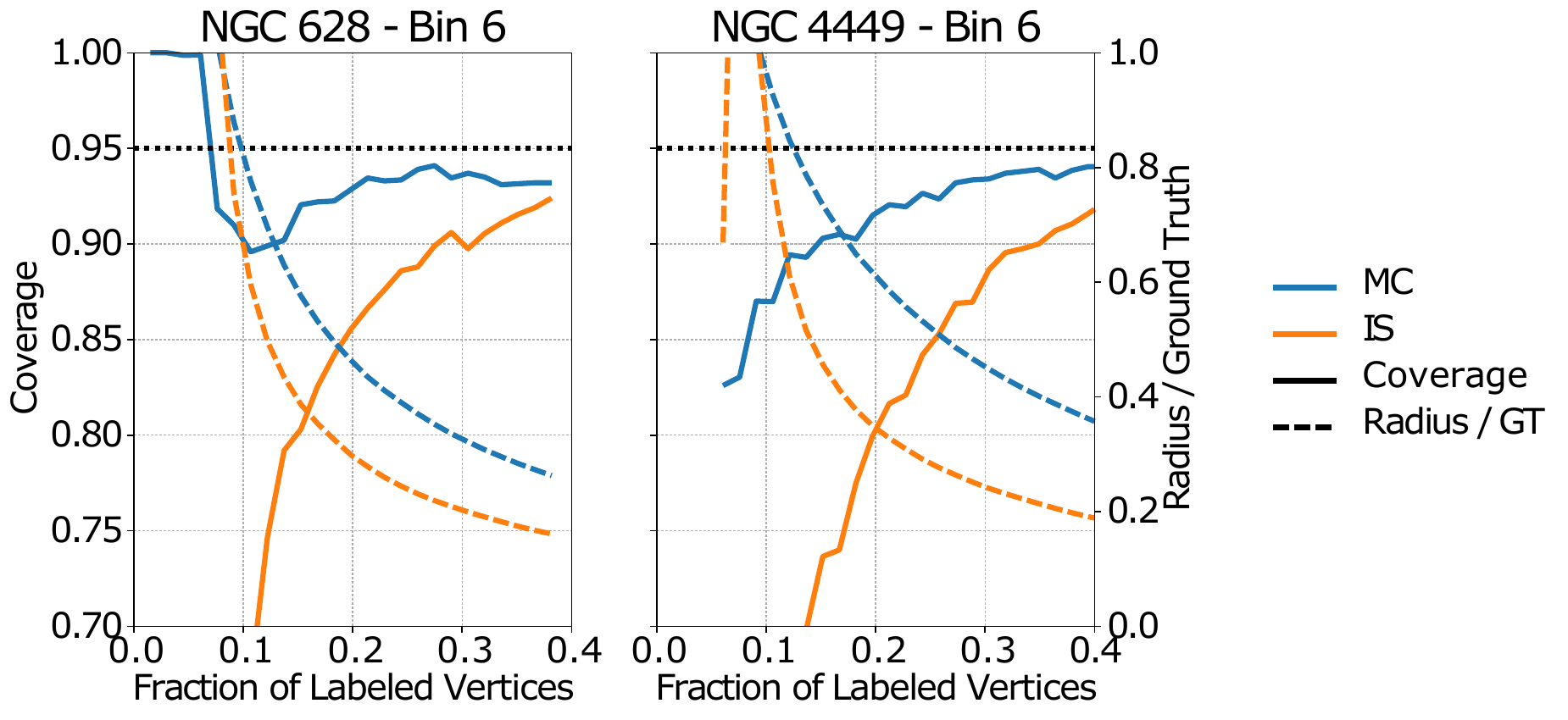}
    \includegraphics[width=0.43\linewidth]{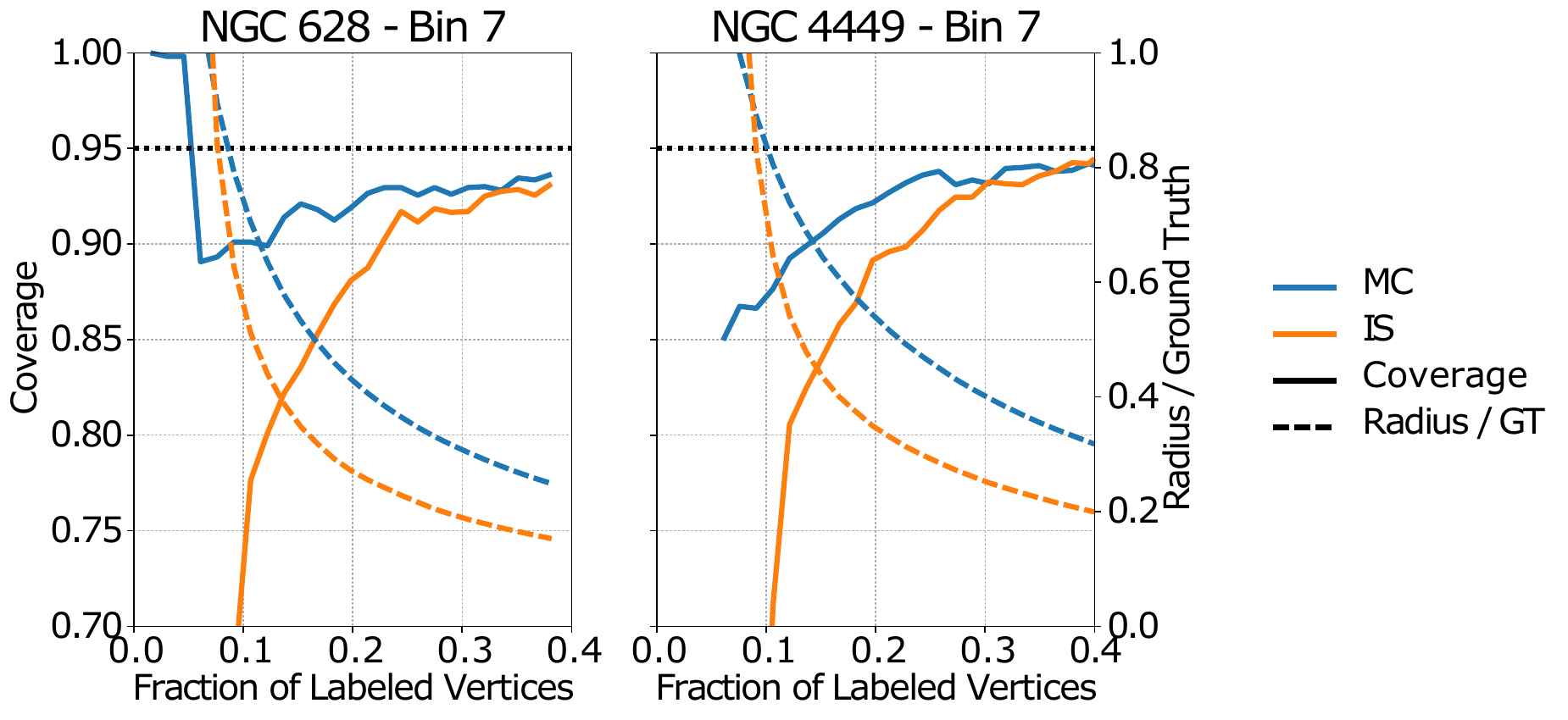}
    \includegraphics[width=0.43\linewidth]{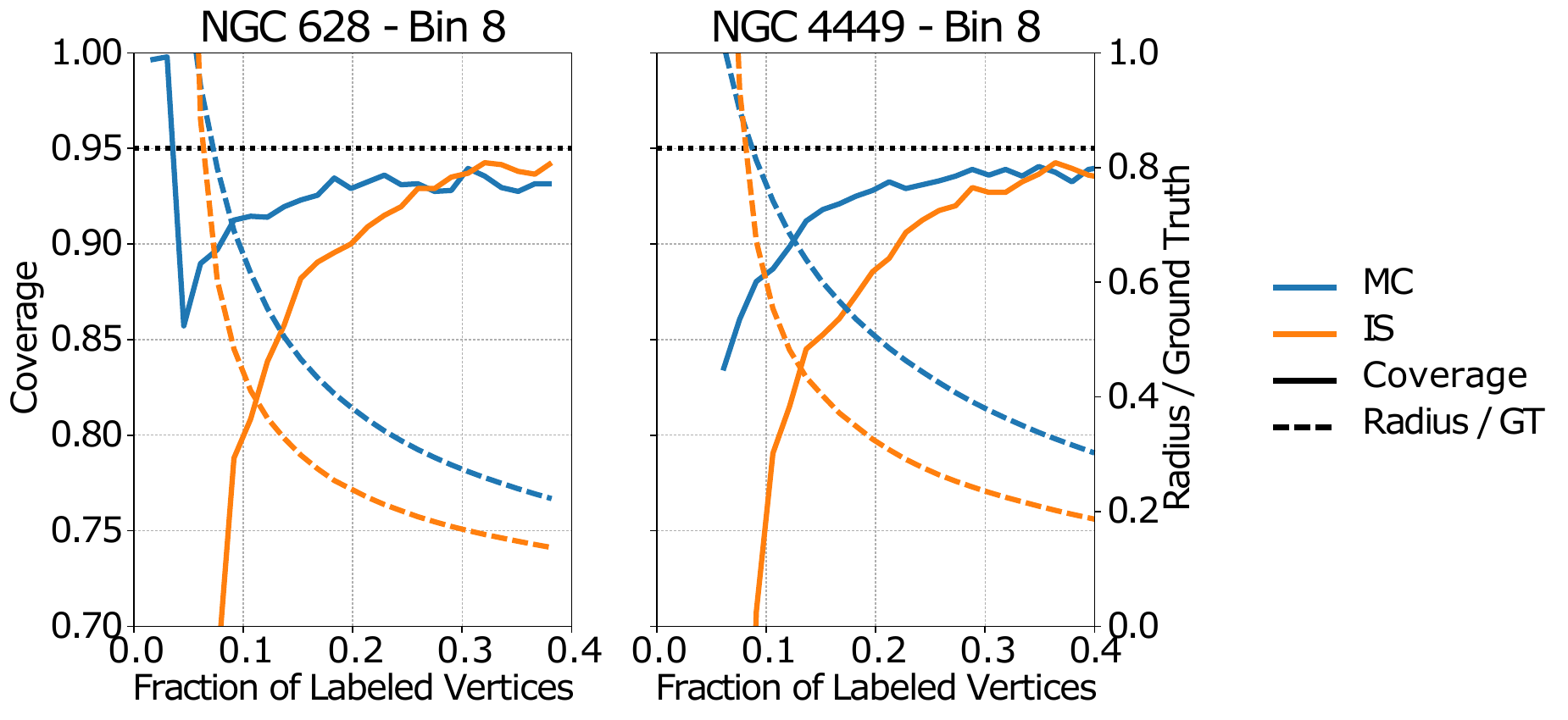}
    \includegraphics[width=0.43\linewidth]{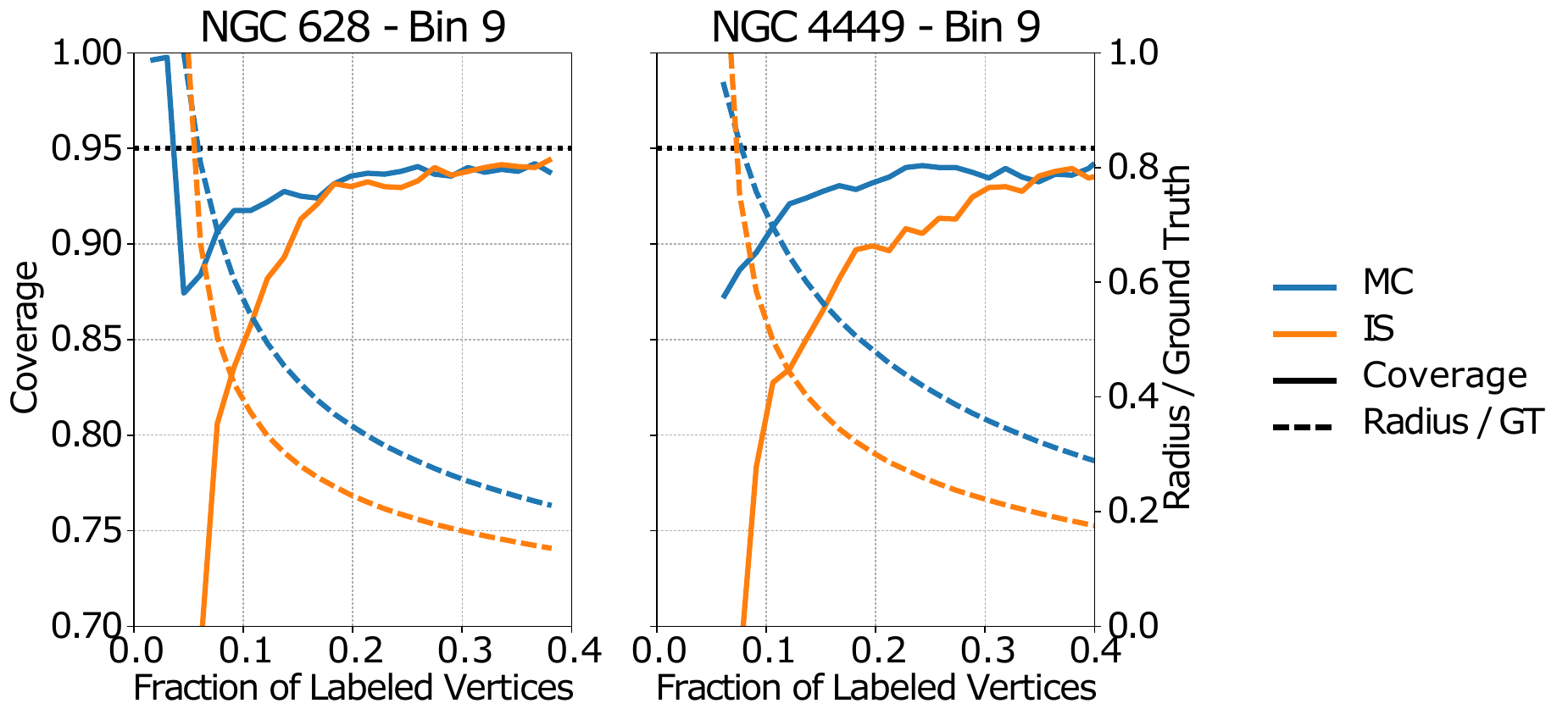}
    \includegraphics[width=0.43\linewidth]{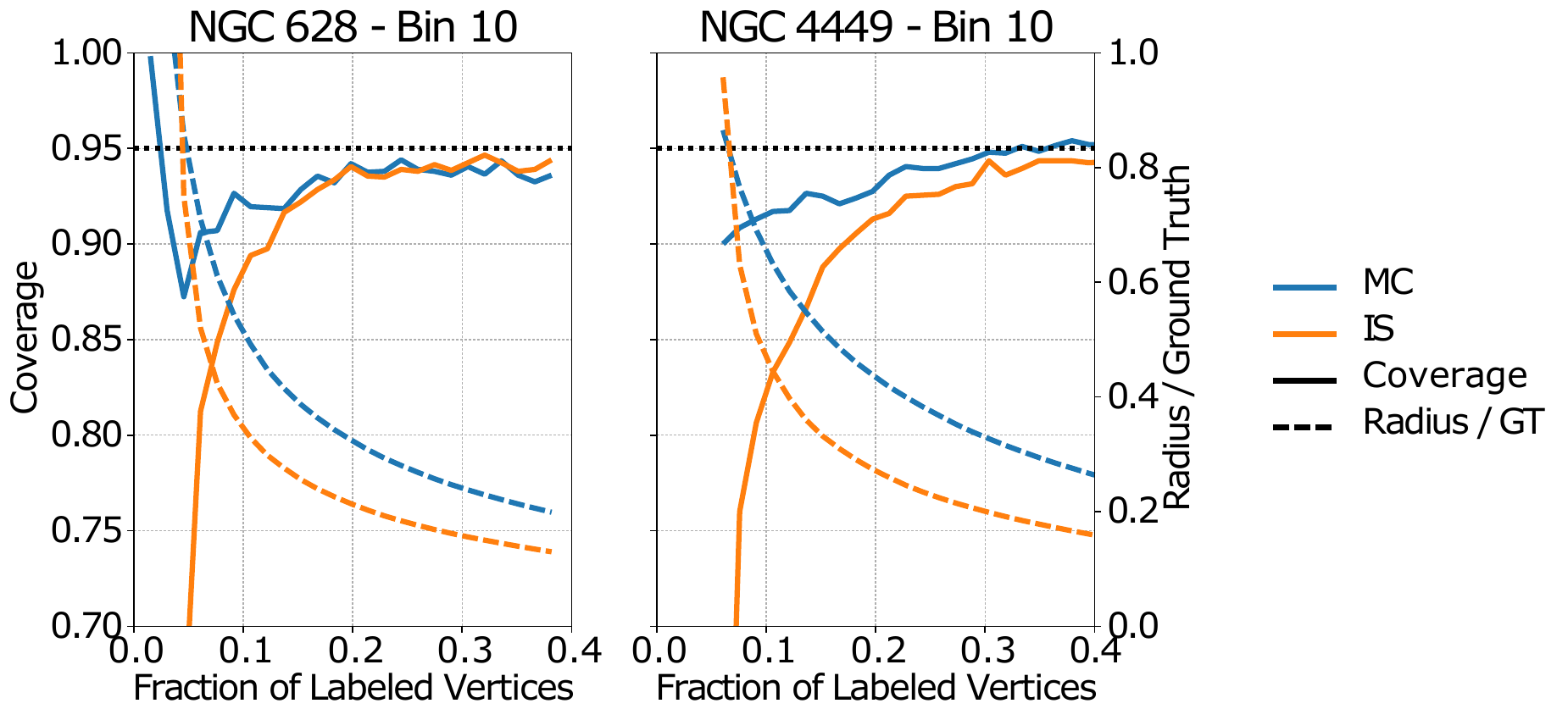}
    \includegraphics[width=0.43\linewidth]{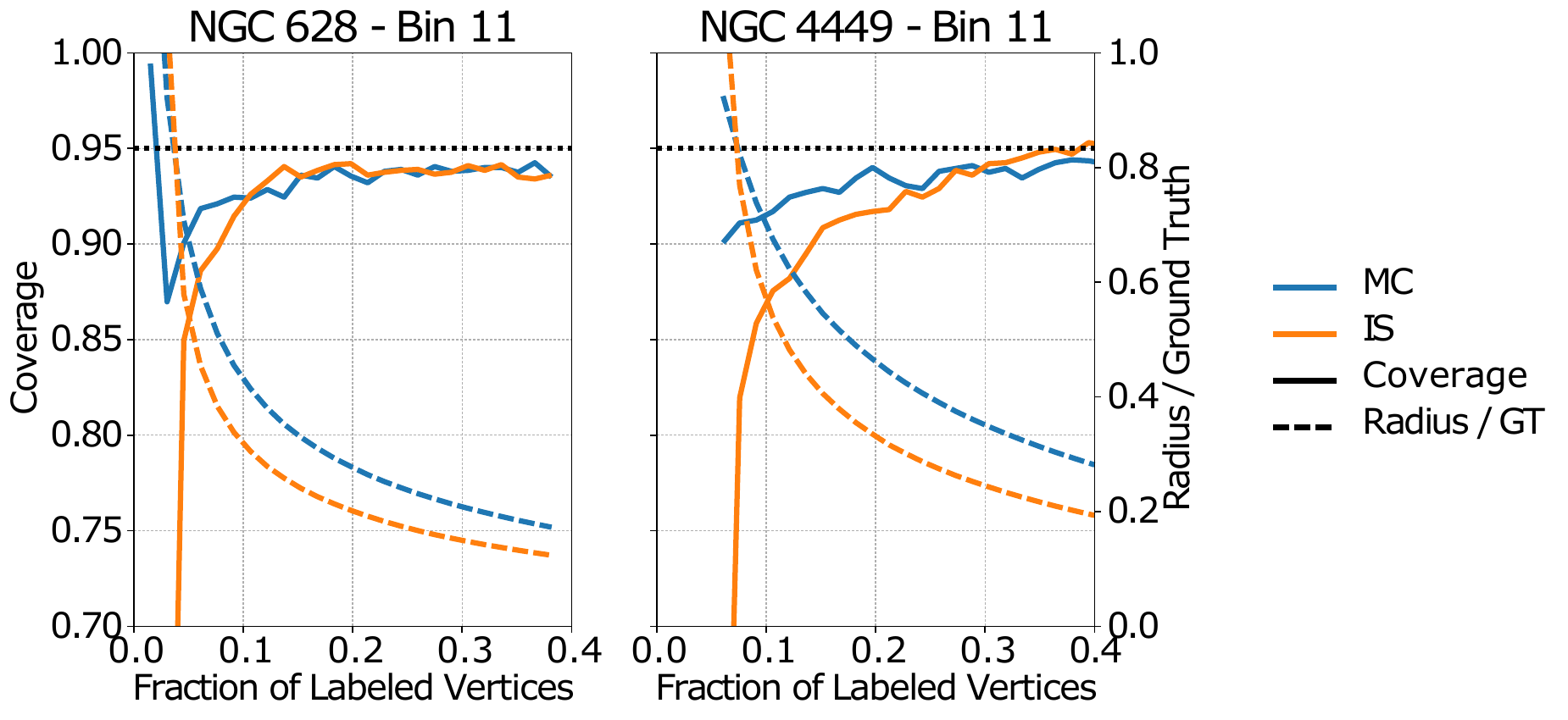}
    \includegraphics[width=0.43\linewidth]{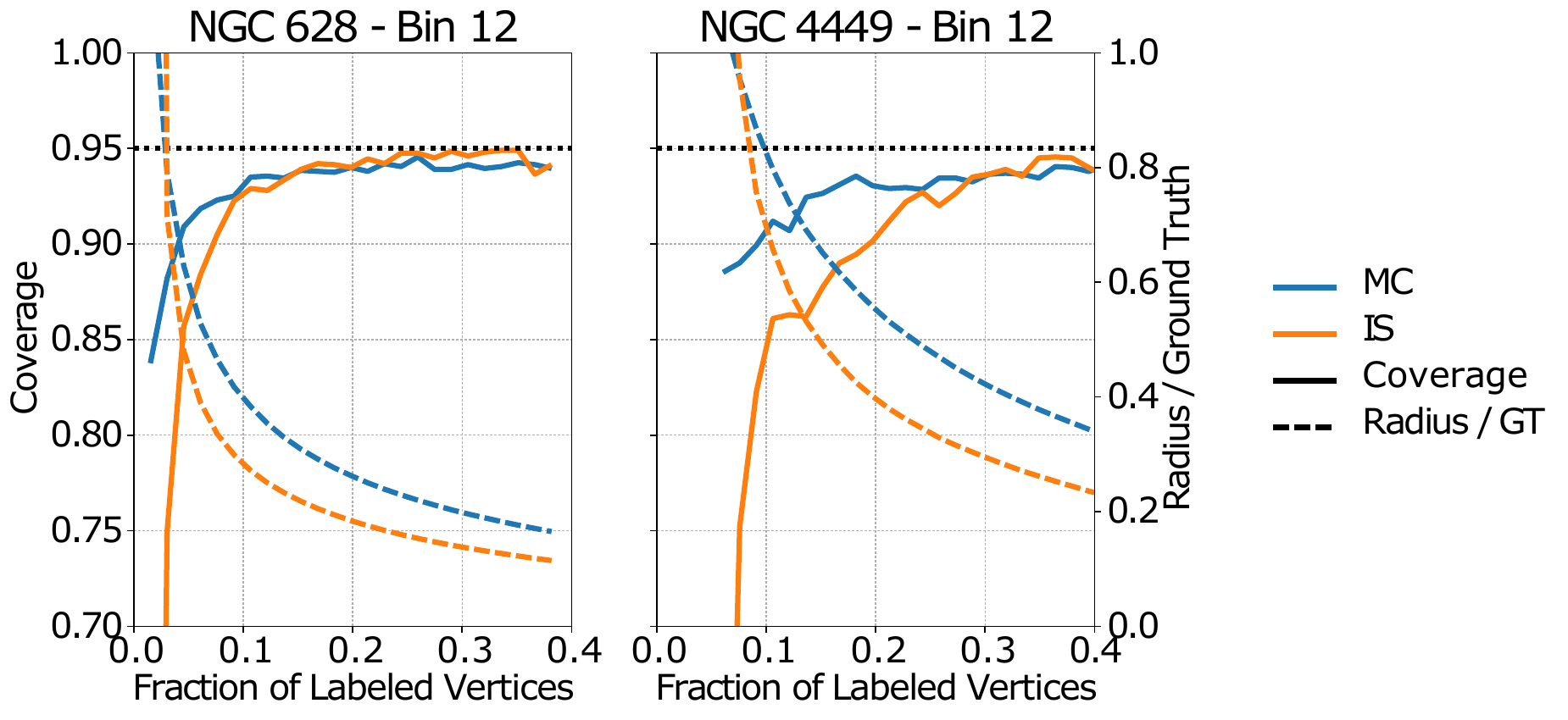}
    \includegraphics[width=0.43\linewidth]{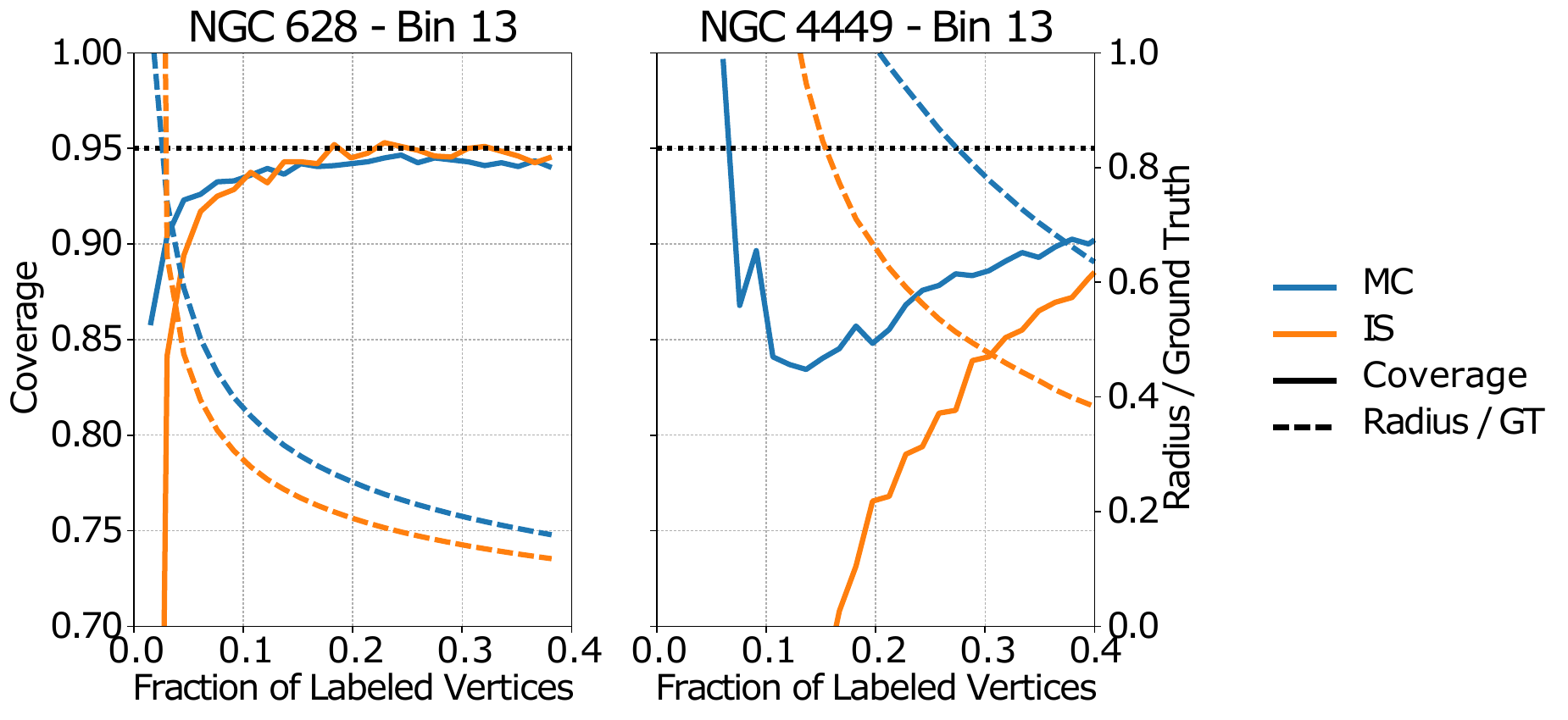}
    \includegraphics[width=0.43\linewidth]{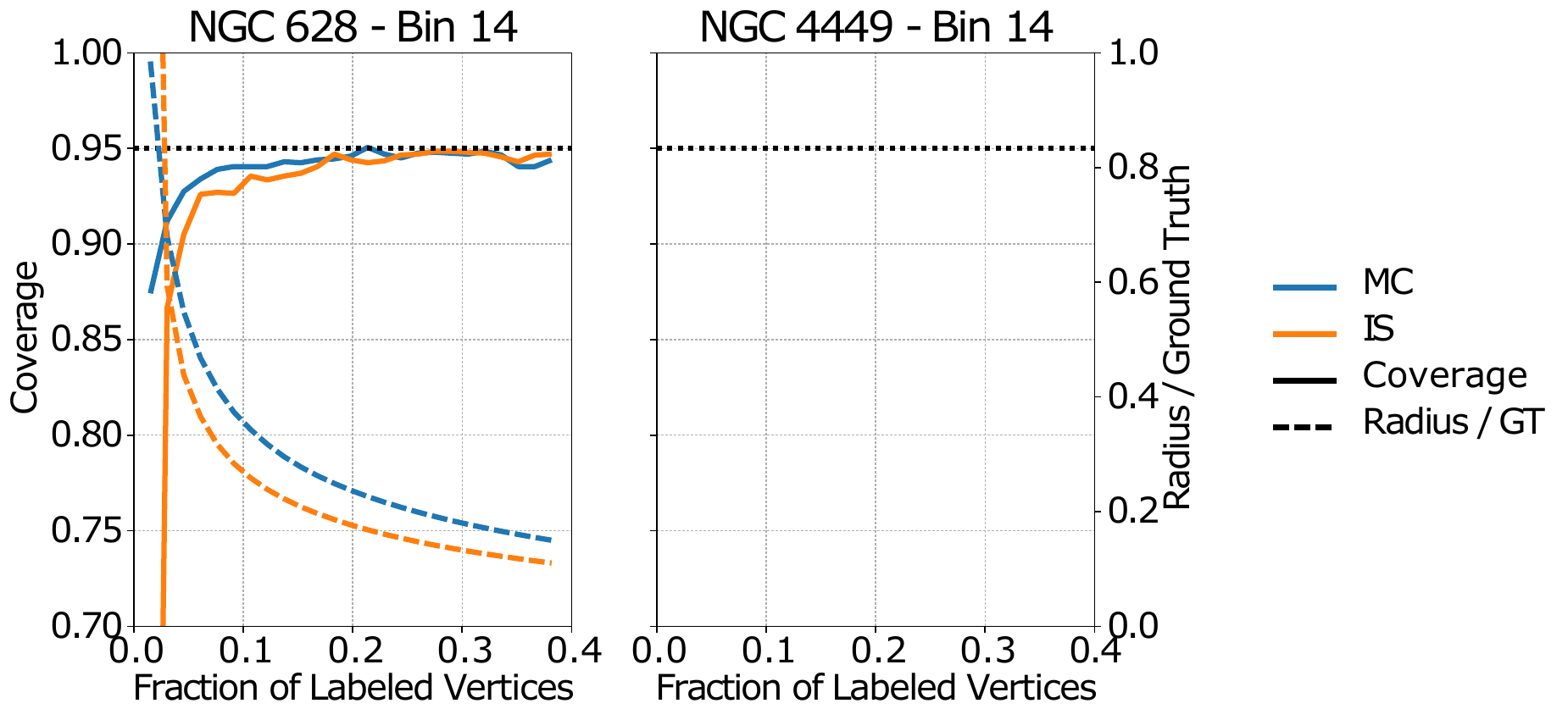}
    \label{fig:all-confidence}
\end{figure}

\vfill

\end{document}